\documentclass[accepted]{article}

\usepackage{microtype}
\usepackage{graphicx}
\usepackage{subfigure}
\usepackage{booktabs} 

\usepackage{hyperref}

\usepackage{icml_template/icml2022}

\usepackage{amsmath}
\usepackage{amssymb}
\usepackage{mathtools}
\usepackage{amsthm}

\usepackage{color}

\usepackage[capitalize,noabbrev]{cleveref}

\theoremstyle{plain}
\newtheorem{theorem}{Theorem}[section]

\newtheorem{lemma}[theorem]{Lemma}

\theoremstyle{definition}

\theoremstyle{remark}

\usepackage[textsize=tiny]{todonotes}

\usepackage{url}

\usepackage{amsfonts}       
\usepackage{nicefrac}       
\usepackage{xcolor}         
\usepackage{physics}
\usepackage{multirow}
\newcommand{\D}{\mathcal{D}}
\newcommand{\Dp}{\mathcal{D}_{x}^{\text{aug}}}
\newcommand{\Dm}{\mathcal{D}_{\backslash x}^{\text{aug}}}
\newcommand{\Dadv}{\mathcal{D}_{\backslash x}^{\text{adv}}}
\DeclareMathOperator*{\argmin}{arg\,min}
\DeclareMathOperator*{\E}{\mathbb{E}}
\usepackage{boldline}
\newcolumntype{?}{!{\vrule width 1.2pt}}
\usepackage{color}

\usepackage{wrapfig}
\newcommand{\M}{M}

\icmltitlerunning{Revisiting Contrastive Learning through the Lens of \\Neighborhood Component Analysis: an Integrated Framework}

\begin{document}

\twocolumn[
\icmltitle{Revisiting Contrastive Learning through the Lens of\\ Neighborhood Component Analysis: an Integrated Framework}




\begin{icmlauthorlist}
\icmlauthor{Ching-Yun Ko}{MIT}
\icmlauthor{Jeet Mohapatra}{MIT}
\icmlauthor{Sijia Liu}{MSU}
\icmlauthor{Pin-Yu Chen}{IBM}
\icmlauthor{Luca Daniel}{MIT}
\icmlauthor{Tsui-Wei Weng}{UCSD}
\end{icmlauthorlist}


\icmlaffiliation{MIT}{MIT}
\icmlaffiliation{MSU}{MSU}
\icmlaffiliation{IBM}{IBM Research AI}
\icmlaffiliation{UCSD}{UCSD}

\icmlcorrespondingauthor{Ching-Yun Ko}{cyko@mit.edu}

\icmlkeywords{Machine Learning, ICML}

\vskip 0.3in
]



\printAffiliationsAndNotice{}  

\begin{abstract}
As a seminal tool in self-supervised representation learning, contrastive learning has gained unprecedented attention in recent years. In essence, contrastive learning aims to leverage pairs of positive and negative samples for representation learning, which relates to exploiting neighborhood information in a feature space. By investigating the connection between contrastive learning and neighborhood component analysis (NCA), we provide a novel stochastic nearest neighbor viewpoint of contrastive learning and subsequently propose a series of contrastive losses that outperform the existing ones. Under our proposed framework, we show a new methodology to design integrated contrastive losses that could simultaneously achieve good accuracy and robustness on downstream tasks. With the integrated framework, we achieve up to 6\% improvement on the standard accuracy and 17\% improvement on the robust accuracy.

\end{abstract}

\section{Introduction}
\label{sec:intro}
Contrastive learning has drawn much attention and has become one of the most effective representation learning techniques recently. The contrastive paradigm~\cite{oord2018representation,wu2018unsupervised,He2020Momentum,chen2020simple,chuang2020debiased,Grill2020Bootstrap} constructs an objective for embeddings based on an assumed semantic similarity between positive pairs and dissimlarity between negative pairs, which stems from instance-level classification~\cite{dosovitskiy2015discriminative,bojanowski2017unsupervised,wu2018unsupervised}.
Specifically, the contrastive loss $\mathcal{L}_{\text{CL}}$ \cite{oord2018representation,chen2020simple} is defined as
$\E_{\substack{x\sim \D, \\x^+\sim \D_{x}^+, \\x_i^-\sim \D_{x}^-}} \small\left[-\log\frac{e^{f(x)^Tf(x^+)}}{e^{f(x)^Tf(x^+)}+\sum\limits_{i=1}\limits^N e^{f(x)^Tf(x^-_i)}}\right]$
where, for an input data sample $x$, $(x,x^+)$ denotes a positive pair and $(x,x^-)$ denotes a negative pair. 
The function $f$ is an encoder parameterized by a neural network and the number of negative pairs $N$ is typically treated as a hyperparameter. Note that the contrastive loss can encode the inputs and keys by different encoders if one considers the use of memory bank or momentum contrast~\cite{wu2018unsupervised,He2020Momentum,Chen2020Improved}. In this work, we will focus on the paradigm proposed in \cite{wang2015unsupervised,ye2019unsupervised,chen2020simple} which has demonstrated competitive results in representation learning. 
\begin{figure*}[t!]
    \centering
    \begin{subfigure}[CIFAR100]{
    \includegraphics[width=0.48\textwidth]{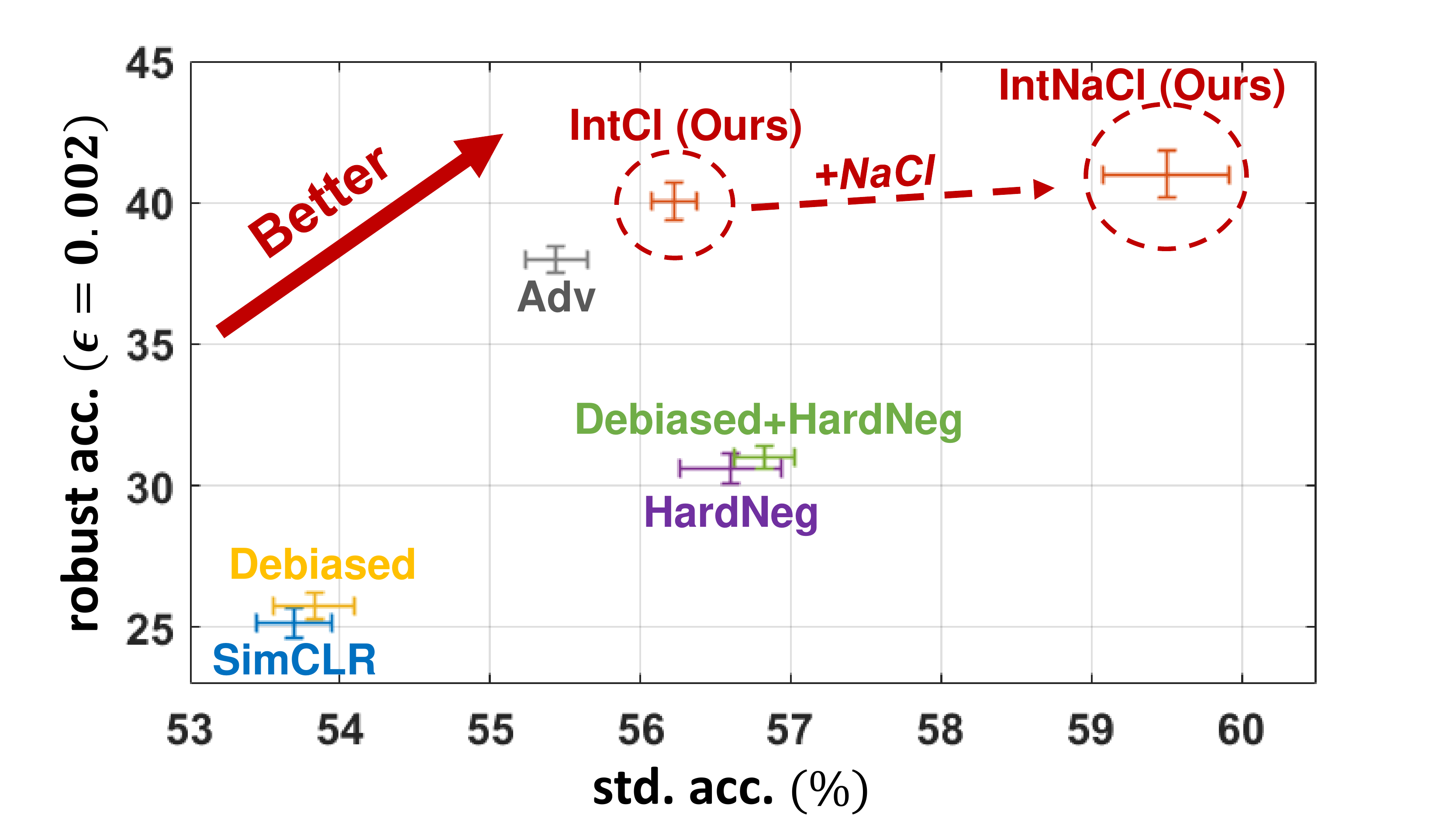}
    \label{subfig:cifar100_2d}}
    \end{subfigure}
    \begin{subfigure}[CIFAR10]{
    \includegraphics[width=0.48\textwidth]{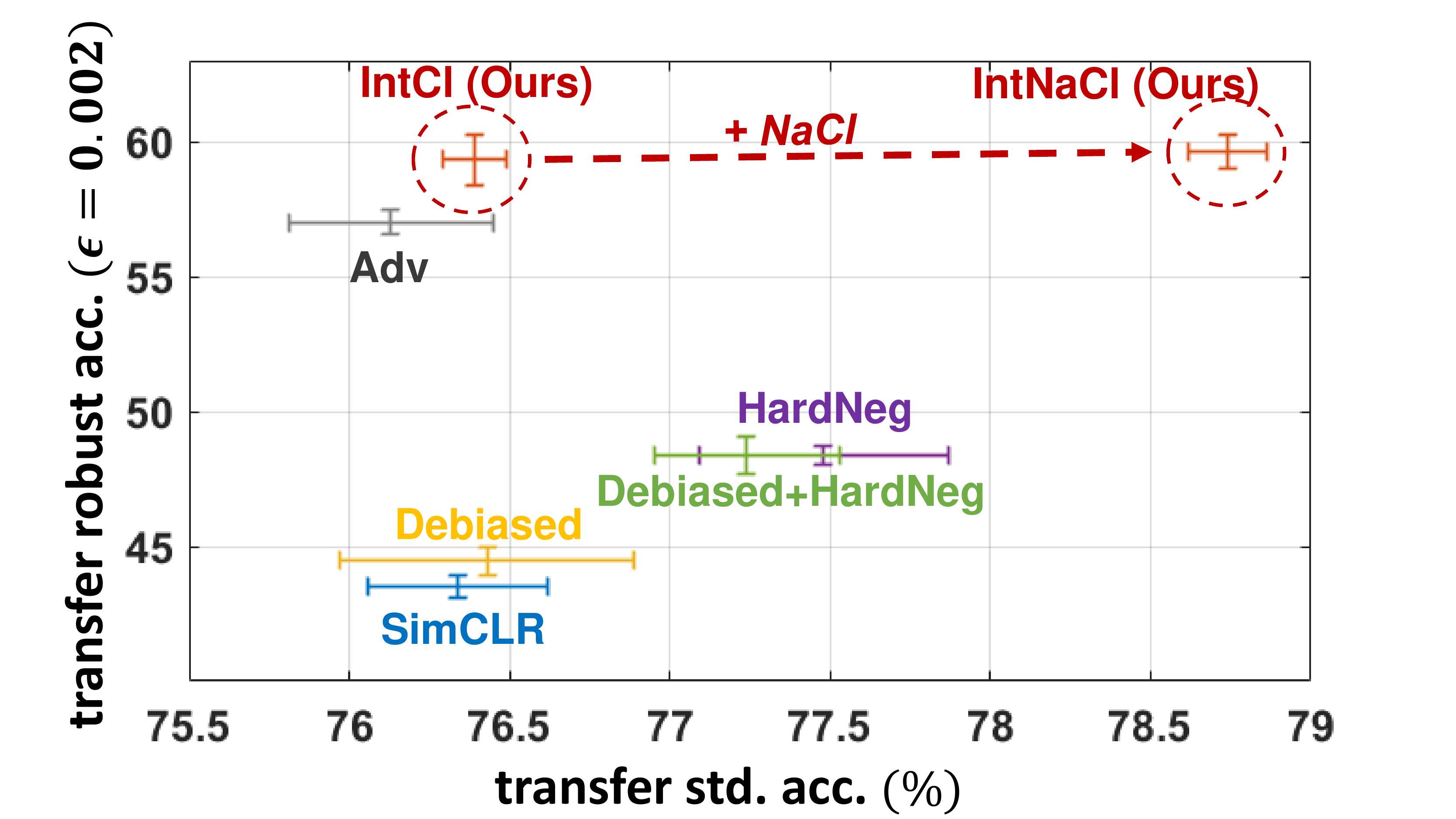}
    \label{subfig:cifar10_2d}}
    \end{subfigure}
    \caption{The performance of existing methods and our proposal (IntNaCl \& IntCl) in terms of their standard accuracy (x-axis) and robust accuracy under FGSM attacks $\epsilon=0.002$ (y-axis). The transfer performance refers to fine-tuning a linear layer for CIFAR10 with representation networks trained on CIFAR100.}
    \label{fig:cifars}
\end{figure*}

When constructing loss $\mathcal{L}_{\text{CL}}$, ideally, one draws $x^+$ from the data distribution $\D_{x}^+$ that characterizes the semantically-\textit{similar} (i.e., \textit{positive}) samples to $x$; similarly, one wants to draw $x^-$ from $\D_{x}^-$ that characterizes the semantically-\textit{dissimilar} (\textit{negative}) samples. 
However, the definition of semantically-\textit{similar} and semantically-\textit{dissimilar} is heavily contingent on downstream tasks: an image of a cat can be considered semantically similar to that of a dog if the downstream task is to distinguish between animal and non-animal classes. Without the knowledge of downstream tasks, $\D_{x}^+$ and $\D_{x}^-$ are hard to define. 
To provide a surrogate of measuring similarity, current mainstream contrastive learning algorithms~\cite{He2020Momentum,chen2020simple,Chen2020Improved,Grill2020Bootstrap} typically build up $\D_{x}^+$ by considering data augmentation $\Dp$ of a data sample $x$. In the meantime, $\D_{x}^-$ is approximated by the joint distribution $\D$ or $\Dm:=\cup_{x'\in\D\backslash\{x\}}\D_{x'}^{\text{aug}}$, and the resulting contrastive loss is known as $\mathcal{L}_{\text{SimCLR}}$ which was proposed in~\cite{chen2020simple}: 
\begin{align}
\label{eqn:biased}
    & (\text{SimCLR loss }\mathcal{L}_{\text{SimCLR}}) \nonumber\\
    & \E_{\substack{x\sim \D, \\x^+\sim \Dp, \\x_i^-\sim \Dm}} \left[-\log\frac{e^{f(x)^Tf(x^+)}}{e^{f(x)^Tf(x^+)}+\sum\limits_{i=1}\limits^N e^{f(x)^Tf(x^-_i)}}\right].
\end{align}
Although this formulation seems to put no assumptions on the downstream task classes, we find that there are in fact implicit assumptions on the class probability prior of the downstream tasks. Specifically, we formally establish the connection between the Neighborhood Component Analysis (NCA) and the unsupervised contrastive learning in this paper for the first time (to our best knowledge). Inspired by this interesting relationship to NCA, we further propose two new contrastive loss (named NaCl) which outperform existing paradigm. Furthermore, by inspecting the robust accuracy of several existing methods (e.g., Figure~\ref{fig:cifars}'s y-axis, the classification accuracy when inputs are corrupted by crafted perturbations), one can see the insufficiency of existing methods in addressing robustness. Thus, we propose a new integrated contrastive framework (named IntNacl and IntCl) that accounts for \textit{both} the standard accuracy and adversarial cases: our proposed method's performance remains in the desired upper-right region (circled) as shown in Figure~\ref{fig:cifars}. A conceptual illustration of our proposals is given in Figure~\ref{fig:concept}.

We summarize our main contributions as follows: 
\begin{itemize}
    \item We establish the relationship between contrastive learning and NCA, and propose two new contrastive loss dubbed \textbf{NaCl} (Neighborhood analysis Contrastive loss). We provide theoretical analysis on NaCl and show better generalization bounds over the baselines;
    \item Building on top of NaCl, we propose a generic framework called Integrated contrastive learning (\textbf{IntCl} and \textbf{IntNaCl}) where we show that the spectrum of recently-proposed contrastive learning losses~\cite{chuang2020debiased,robinson2021contrastive,ho2020contrastive} can be included as special cases of our framework;
    \item We provide extensive experiments that demonstrate the effectiveness of IntNaCl in improving standard accuracy and robust accuracy. Specifically, IntNaCl improves upon literature~\cite{chen2020simple,chuang2020debiased,robinson2021contrastive,ho2020contrastive} by 3-6\% and 4-16\% in CIFAR100 standard and robust accuracy, and 2-3\% and 3-17\% in CIFAR10 standard and robust accuracy, respectively. 
\end{itemize}

\begin{figure}
    \centering
    \includegraphics[width=0.49\textwidth]{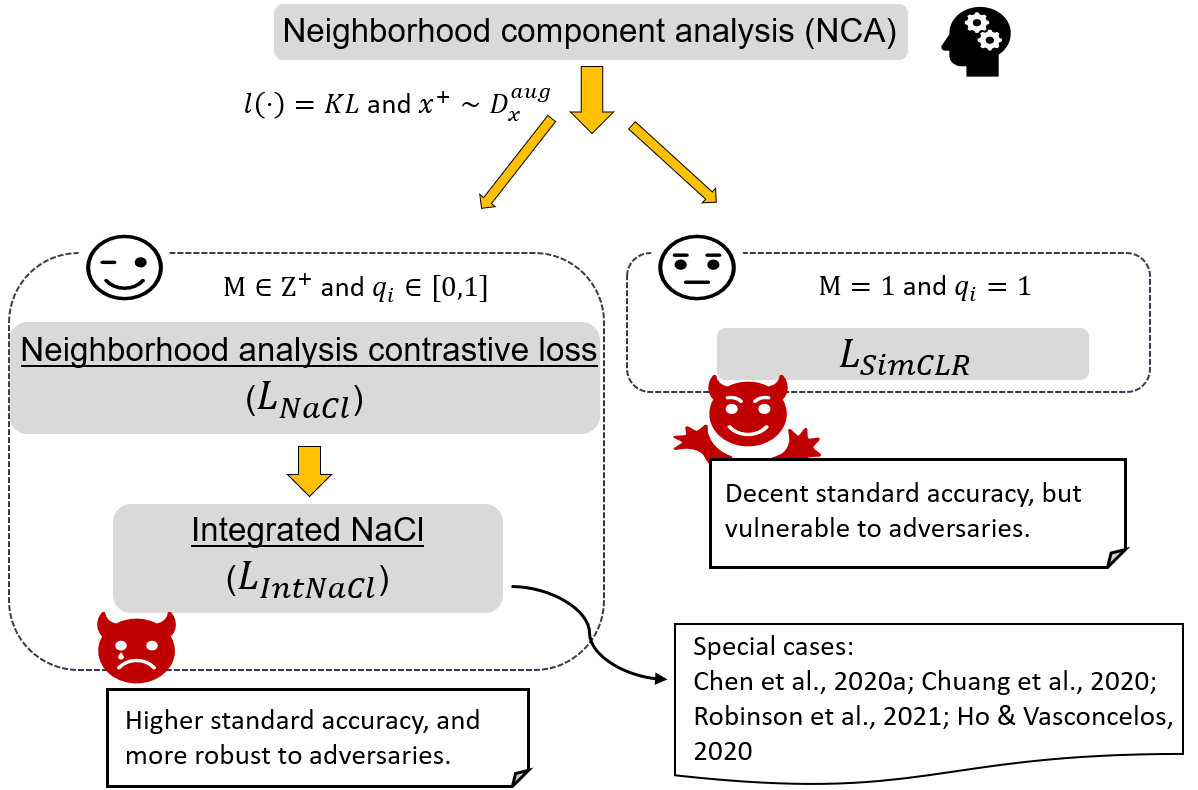}
    \caption{A conceptual illustration of the relationships among NCA, $\mathcal{L}_{\text{SimCLR}}$, and our proposals.}
    \label{fig:concept}
\end{figure}

\section{Related Work}
\label{sec:related}


\paragraph{Contrastive learning.} In the early work of \cite{dosovitskiy2015discriminative}, authors treat every individual image in a dataset as belonging to its own class and do multi-class classification tasks under the setting. However, this regime will soon become intractable as the size of dataset increases. 
To cope with this,~\cite{wu2018unsupervised} designs a memory bank for storing seen representations (keys) and utilize noise contrastive estimation~\cite{gutmann2010noise,Mnih2012Fast,jozefowicz2016exploring,oord2018representation} for representation comparisons. 
\cite{He2020Momentum} and \cite{Chen2020Improved} further improve upon \cite{wu2018unsupervised} by storing keys inferred from a momentum encoder other than the representation encoder for $x$.
To further reduce the computational cost, besides the practical tricks introduced in SimCLR~\cite{chen2020simple} (e.g. stronger data augmentation scheme and projector heads), authors of SimCLR get rid of the memory bank and instead makes use of other samples from the same batch to form contrastive pairs. 

In the rest of this paper, we will focus on the setups of SimCLR and the related follow up work~\cite{chuang2020debiased,robinson2021contrastive,ho2020contrastive} due to computational efficiency. A temperature scaling hyperparameter $t$ is normally used in contrastive learning to tune the radius of the hypersphere that representations lie in. For better readability, without loss of generality, we let $t = 1$ in all equations. We let $g_0(x,\{x_i^-\}^N)$ denote the negative term $\frac{1}{N}\sum_{i=1}^N e^{f(x)^Tf(x^-_i)}$, where the subscript $i$ identifies the summation index and the superscript $N$ identifies the summation limits. We omit the subscript $i$ when the sample index is one dimensional (e.g. $x^-_i$ has 1-D index, $x^-_{ij}$ has 2-D index). 
Then $\mathcal{L}_{\text{SimCLR}}$ in Equation \eqref{eqn:biased} can be re-written as 
\begin{align}
\label{eqn:simclr}
    & (\text{Re-written SimCLR loss }\mathcal{L}_{\text{SimCLR}}) \nonumber\\
    & \E_{\substack{x\sim \D, \\x^+\sim \Dp, \\x_i^-\sim \Dm}}
    \left[-\log\frac{e^{f(x)^Tf(x^+)}}{e^{f(x)^Tf(x^+)}+N  g_0(x,\{x_i^-\}^N)}\right].
\end{align}


\paragraph{Designing \textit{negative} pairs in contrastive learning.} Several works~\cite{saunshi2019theoretical,chuang2020debiased} have come to the awareness of the sampling bias of negative pairs in Equation~\ref{eqn:simclr}. Specifically, if the negative samples are sampled from $\D$, we will receive with $1/K$ probability a positive sample in a $K$-class classification task with balanced classes, hence biasing the contrastive loss.
To overcome this issue, \cite{chuang2020debiased} proposes a \textit{de-biased} constrastive loss to mitigate the sampling bias by explicitly including the class probability prior on the downstream tasks (e.g., with probability $0.1$, $x_i^-$ contains a positive example in CIFAR10), and tune the prior $\tau^+$ as a hyperparameter. We denote the loss from \cite{chuang2020debiased} as $\mathcal{L}_{\text{Debiased}}$ and the full equation is shown below:
\begingroup
\begin{align}
    & (\text{Debiased loss }\mathcal{L}_{\text{Debiased}}) \nonumber\\
    & \E_{\substack{x\sim \D, \\x^+\sim \Dp, \\v_j\sim \Dp, \\u_i\sim \Dm}}
    \left[-\log\frac{e^{f(x)^Tf(x^+)}}{e^{f(x)^Tf(x^+)}+N  g_1(x,\{u_i\}^n,\{v_j\}^m)}\right],
    \label{eqn:debiased}
\end{align}
\endgroup
where the estimator $g_1(x,\{u_i\}^n,\{v_j\}^m)$ is defined by 
$\max\{\frac{\sum_{i=1}^n e^{f(x)^Tf(u_i)}}{(1-\tau^+)n}- \frac{\tau^+\sum_{j=1}^m e^{f(x)^Tf(v_j)}}{(1-\tau^+)m},e^{-1/t}\}$
and $n$ and $m$ represents the numbers of sampled points in $\Dm$ and $\Dp$ for the re-weighted negative term, $\tau^+$ is the class probability prior,
and $t$ is the temperature hyperparameter. Recently,~\cite{robinson2021contrastive} proposes to weigh sample pairs through the cosine distance in the estimator $g_1(x,\{u_i\}^n,\{v_j\}^m)$ based on $\mathcal{L}_{\text{Debiased}}$, and we denote their approach as $\mathcal{L}_{\text{Debiased+HardNeg}}$, 
\begingroup
\begin{align} 
    & (\text{Debiased+HardNeg loss }\mathcal{L}_{\text{Debiased+HardNeg}}) \nonumber\\
    & \E_{\substack{x\sim \D, \\x^+\sim \Dp, \\v_j\sim \Dp, \\u_i\sim \Dm}}
    \left[-\log\frac{e^{f(x)^Tf(x^+)}}{e^{f(x)^Tf(x^+)}+N  g_2(x,\{u_i\}^n,\{v_j\}^m)}\right],
    \label{eqn:harddebiased}
\end{align}
\endgroup
where the estimator $g_2(x,\{u_i\}^n,\{v_j\}^m)$ is defined by 
$\max\{\frac{\sum_{i=1}^n \kappa_i^{\beta+1}}{(1-\tau^+)\sum_{i=1}^n \kappa_i^\beta}-\frac{\tau^+\sum_{j=1}^m e^{f(x)^Tf(v_j)}}{(1-\tau^+)m},e^{-1/t}\}$
and $\kappa_i =  e^{f(x)^Tf(u_i)}$. A typical choice of $n$ and $m$ are $n=N$ and $m=1$, and the hyperparameter $\tau^+$ in $g_2$ is exactly the same as that in $g_1$ whereas the hyperparameter $\beta$ controls the weighting mechanism. Specifically, when $\tau^+=0$, we denote $\mathcal{L}_{\text{Debiased+HardNeg}}$ as $\mathcal{L}_{\text{HardNeg}}$; when $\beta=0$, Equation \eqref{eqn:harddebiased} degenerates to Equation \eqref{eqn:debiased} which is $\mathcal{L}_{\text{Debiased}}$. 

\paragraph{Designing \textit{positive} pairs in contrastive learning.}
Instead of modifying the negative pairs,
another direction is to design the 
positive pairs \cite{ho2020contrastive,kim2020adversarial}. Specifically, authors of \cite{ho2020contrastive} define the concept of \textit{adversarial examples} in the regime of representation learning as the positive sample $x^{\text{adv}}$ that maximizes $\mathcal{L}_{\text{SimCLR}}$ in Equation \eqref{eqn:simclr} within a pre-specified perturbation magnitude $\epsilon$. The resulting loss function is denoted as $\mathcal{L}_{\text{Adv}}$:
\begin{align}
    & (\text{Adversarial loss }\mathcal{L}_{\text{Adv}}) \nonumber\\
    & \E_{\substack{x\sim \D, \\x^+\sim \Dp, \\x_{i_1}^-\sim \Dm, \\x_{i_2}^-\sim \Dadv}}
    \left[-\log\frac{e^{f(x)^Tf(x^+)}}{e^{f(x)^Tf(x^+)}+N  g_0(x,\{x_{i_1}^-\}^N)}\right. \nonumber\\
    \;\;\; &~\  \left. - \alpha \log\frac{e^{f(x)^Tf(x^{\text{adv}})}}{e^{f(x)^Tf(x^{\text{adv}})}+N  g_0(x,\{x_{i_2}^-\}^N)}\right],
    \label{eqn:adversarial}
\end{align}
where the $\Dadv$ is defined by $\cup_{x'\in\D\backslash\{x\}} x'\cup x'^{,\text{adv}}$.
Notably, 
one can adjust the importance of the adversarial term by tuning $\alpha$ in Equation \eqref{eqn:adversarial}. 

\paragraph{Adversarial Robustness.}
Despite neural networks' supremacy in achieving impressive performance, they have been proved vulnerable to human-imperceptible perturbations~\cite{goodfellow2014explaining,szegedy2014intriguing,nguyen2015deep,moosavi2016deepfool}. In the supervised learning setting, an adversarial perturbation $\delta$ is defined to render inconsistent classification result of the input $x$: $r(x+\delta)\neq r(x)$, where $r$ is a neural network classifier. A stronger adversarial attack means it can find $\delta$ with higher success attack rate under the same $\epsilon$-budget ($\|\delta\|_p\leq\epsilon$). One of the most popular and classical attack algorithms is FGSM~\cite{goodfellow2014explaining}, where with a fixed perturbation magnitude $\epsilon$, FGSM finds adversarial perturbation by 1-step gradient descent.
Another popular attack method we consider in this paper is PGD~\cite{madry2018towards}, which assembles the iterative-FGSM~\cite{kurakin2016adversarial} but with different initializations and learning rate constraints.

\section{Two New NCA-inspired Contrastive Losses and an Integrated Framework}
\label{sec:methodology0}
In this section, we first derive a connection between Neighborhood Component Analysis (NCA) \cite{goldberger2004neighbourhood} and the unsupervised contrastive learning loss in Section~\ref{subsec:nca}. Inspired by our result in Section~\ref{subsec:nca}, we propose two new NCA-inspired contrastive losses in Section~\ref{subsec:NaCl}, which we refer to as \textbf{N}eighborhood \textbf{a}nalysis  \textbf{C}ontrastive \textbf{l}oss (\textbf{NaCl}). 
To address a lack of robustness in existing contrastive losses, in Section~\ref{sec:tricks}, we propose a useful framework \textbf{IntNaCl} that integrates NaCl and a robustness-promoting loss. A summary of definitions is given as Table~\ref{tbl:summary_2}.


\subsection{Bridging from supervised NCA to unsupervised contrastive learning: a new finding}
\label{subsec:nca}
NCA is a supervised learning algorithm concerned with learning a quadratic distance metric with the matrix $A$ such that the performance of nearest neighbour classification is maximized. 
Notice that the set of neighbors for a data point is a function of transformation $A$. However, it can remain unchanged as $A$ changes 
within a certain range. Therefore the leave-one-out classification performance can be a piecewise-constant function of $A$ and hence non-differentiable.
To overcome this, the optimization problem is generally given using the concept of stochastic nearest neighbors. In the stochastic nearest neighbor setting, nearest neighbor selection is regarded as a random event, where the probability that point $x_j$ is selected as the nearest neighbor for $x_i$ is given as $p(x_j\mid x_i)$ with
\begin{align}
    p_{ij} := p(x_j\mid x_i) &= \frac{e^{-\norm{Ax_i - Ax_j}^2}}{\sum_{k\neq i}e^{-\norm{Ax_i - Ax_k}^2}},~  j\neq i.\label{eqn:pij}
\end{align}

Let $c_i$ denote the label of $x_i$, in the leave-one-out classification loss, the probability a point is classified correctly is given as $p_i = \sum_{j\mid c_j = c_i} p_{ij}$, where $\{j\mid c_j = c_i\}$ defines an index set in which all points $x_j$ belong to the same class as point $x_i$. We use $M$ to denote the cardinality of this set. By the definition of $c_i$, the probability $x_i$'s label is $c_i$ is given as $q_i$, which is exactly 1\footnote{For every data point, $p$ and $q$ are defined differently with their supports being the class index. For every sample $x$, $q_i$ is the ground truth probability of class labels and $p_i$ is the prediction probability.}. Thus the optimization problem can be written as $\min_A \sum_{i=1}^n \ell(q_i, \sum_{j \mid c_j = c_i} p_{ij})$.
This learning objective then naturally maximizes the expected accuracy of a 1-nearest neighbor classifier. Two popular choices for $\ell(\cdot)$ are the total variation distance and the KL divergence. In the seminal paper of~\cite{goldberger2004neighbourhood}, the authors showed both losses give similar results, thus we will focus on the KL divergence loss in this work.
For $\ell(\cdot) = $ KL, the relative entropy from $p$ to $q$ is $D_{\text{KL}}(q\|p)=\sum_i- q_i\log\frac{p_i}{q_i}=\sum_i -\log p_i$ when $q_i=1$. By plugging in the definition of $p_i = \sum_{j\mid c_j = c_i} p_{ij}$ and Equation~\ref{eqn:pij}, the NCA problem becomes
\begin{align}
\label{eqn:NCA}
    \min_A \sum_{i=1}^n - \log\left(\sum_{j \mid c_j = c_i} \frac{e^{-\norm{Ax_i - Ax_j}^2}}{\sum_{k\neq i}e^{-\norm{Ax_i - Ax_k}^2}}\right).
\end{align}
With the above formulation, we now show how to establish the connection of NCA to the contrastive learning loss. First, by assuming (a) positive pairs belong to the same class and (b) the transformation $Ax$ is instead parametrized by a general function $\frac{f(x)}{\sqrt{2}} := \frac{h(x)}{\sqrt{2}\norm{h(x)}}$, where $h$ is a neural network, we could derive from Equation \eqref{eqn:NCA} to Equation~\eqref{eqn:start} in Appendix A. Next, we show that with some manipulations (details please see Appendix A), below Equation~\eqref{eqn:end} is equivalent to Equation~\eqref{eqn:start}:
\begin{align}
    \min_f \!\! \E_{x\sim \D}&
    \left[- \log\left(\frac{\sum\limits_{j=1}\limits^\M e^{f(x)^Tf(x_{j}^+)}}{\sum\limits_{j=1}\limits^\M e^{f(x)^Tf(x_j^+)} +N  g_0(x,\{x_i^-\}^N)}\right)\right]\!\!.
    \label{eqn:end}
\end{align}
Notice that Equation~\eqref{eqn:end} is a more general contrastive loss where the contrastive loss $\mathcal{L}_{\text{SimCLR}}$ in~\cite{chen2020simple} is a special case with $\M=1, x^+\sim \Dp$:
\begin{align*}
    \min_f \!\! \E_{\substack{x\sim \D, \\x^+\sim \Dp, \\x_i^-\sim \Dm}}
    \left[- \log\left(\frac{ e^{f(x)^Tf(x^+)}}{ e^{f(x)^Tf(x^+)} +N  g_0(x,\{x_i^-\}^N)}\right)\right].
\end{align*}
With the above analysis, two new contrastive losses are proposed based on Equation~\eqref{eqn:end} in the next Section~\ref{subsec:NaCl}. As a side note, as the computation of the loss grows quadratically with the size of the dataset, the current method~\cite{chen2020simple} uses mini batches to construct positive/negative pairs in a data batch of size $N$ to estimate the loss. 

\subsection{Neighborhood analysis Contrastive loss (NaCl)}
\label{subsec:NaCl}
Based on the connection we have built in Section~\ref{subsec:nca}, we discover that the reduction from the NCA formulation to $\mathcal{L}_{\text{SimCLR}}$ assumes
\begin{enumerate}
    \item the expected relative density of positives in the underlying data distribution is $1/N$;
    \item the probability $q_i$ induced by encoder network $f$ is 1.
\end{enumerate}
By relaxing the assumptions individually, in this section, we propose two new contrastive losses. Note that the two neighborhood analysis contrastive losses are designed from orthogonal perspectives, hence they are complementary to each other. We use $\mathcal{L}_{\text{NaCl}}$ to denote these two variant losses: $\mathcal{L}_{\text{NCA}}$ and $\mathcal{L}_{\text{MIXNCA}}$.

\paragraph{(I) Relaxing assumption 1: $\mathcal{L}_{\text{NCA}}$.}
When relating unsupervised SimCLR to supervised NCA, we view two samples in a positive pair as same-class samples. Since in SimCLR, the number of positive pairs $M = 1$, which means that $\{j\mid c_j = c_i\}$ only contains one element. This implies the relative density of positives in the underlying data distribution is $\M/N=1/N$, where $N$ is the data batch size. However, as the expected relative density is task-dependent, it's more reasonable to treat the $\M/N$ ratio as a hyperparameter similar to the class probabilities $\tau^+$ introduced by \cite{chuang2020debiased}. Therefore, we propose the more general contrastive loss $\mathcal{L}_{\text{NCA}}$ which could include more than one element or equivalently $\M\neq 1$:
\begin{align*}
    & (\text{NCA loss }\mathcal{L}_{\text{NCA}}(G=g_0,\M)) \nonumber\\
    &\E_{\substack{x\sim \D, \\x_j^+\sim \Dp, \\x_i^-\sim \Dm}}
    \left[-\log\frac{\sum\limits_{j=1}\limits^\M e^{f(x)^Tf(x_j^+)}}{\sum\limits_{j=1}\limits^\M e^{f(x)^Tf(x_j^+)}+N  g_0(x,\{x_{i}^-\}^N)}\right].
\end{align*}
We further provide the generalization results as follows: if we let $\mathcal{F}$ be a function class, $K$ be the number of classes, $\mathcal{L}_{\textnormal{Sup}}$ be the cross entropy loss of any downstream K-class classification task, $\widehat{\mathcal{L}}_{\substack{\textnormal{NCA}}}(g_0,\M)$ be the empirical NCA loss, $T$ be the size of the dataset, and $\mathcal{R}_\mathcal{S}(\mathcal{F})$ be the empirical Rademacher complexity of $\mathcal{F}$ w.r.t. data sample $\mathcal{S}$, then 
\begin{theorem}
\label{thm:generalization}
With probability at least $1-\delta$, for any $f \in \mathcal{F}$ and $N \geq K-1$,
\begin{align*}
      \mathcal{L}_{\textnormal{Sup}}(\hat{f}) &\leq \mathcal{L}_{\substack{\textnormal{NCA}}}(g_0,\M)(f) \\
      &+ \mathcal{O} \left (\sqrt{\frac{1}{N}} +  \frac{\lambda \mathcal{R}_\mathcal{S}(\mathcal{F})}{T}+  B\sqrt{\frac{\log{\frac{1}{\delta}}}{T}} \right ),
\end{align*}
where $\hat{f}=\argmin_{f \in \mathcal{F}} \widehat{\mathcal{L}}_{\textnormal{NCA}}(g_0,\M)(f)$, $\lambda = \frac{1}{\M}$, and $B = \log N$. 
\end{theorem}
We can see from the term $\lambda$ that $\mathcal{L}_{\text{NCA}}(G=g_0,\M)$ improves upon $\mathcal{L}_{\text{SimCLR}}$ by using a $\M\neq1$. The result extends to $G=g_1$ and for more details please refer to Appendix B.

\paragraph{(II) Relaxing assumption 2: $\mathcal{L}_{\text{MIXNCA}}$.}
To reduce the reliance on the downstream prior, a practical relaxation can be made by allowing neighborhood samples to agree with each other with probability. This translates into relaxing the specification of $q_i=1$ and consider a synthetic data point $x' = \lambda x_i+(1-\lambda) y, y \sim \mathcal{D}$ that belongs to a synthetic class $c_{\lambda, i}$. Assume the probability $x_i$'s label is $c_{\lambda, i}$ is $q_{\lambda,i}=\lambda + (1-\lambda) [c_y = c_i]$, then $q_{\lambda,i}$ should match the probability $p_{\lambda,i} = \sum_{j\mid c_j = c_{\lambda,i}} p_{ij}$, where $\{j\mid c_j = c_{\lambda,i}\}$ is a singleton containing only the index of $x'$, which yields
\begin{align*}
    & (\text{MIXNCA loss }\mathcal{L}_{\text{MIXNCA}}(G=g_0,\M,\lambda)) \nonumber\\
    & \E_{\substack{x\sim \D, \\x^+\sim \Dp, \\x_{i_1}^-, x_{i_2j}^-, x_{j}^-\sim \Dm}}
    \left[-\log\frac{e^{f(x)^Tf(x^+)}}{e^{f(x)^Tf(x^+)}+N  g_0(x,\{x_{i_1}^-\}^N)}\right.
    \nonumber\\
    & - \frac{\lambda}{\M-1}\sum\limits_{j=1}\limits^{\M-1}
    \log \Omega_j - \left.\frac{1-\lambda}{\M-1}\sum\limits_{j=1}\limits^{\M-1}
    \log(1-\Omega_j)\right], 
\end{align*}
where 
$\Omega_j = \frac{e^{f(x)^Tf(\lambda x^+ +(1-\lambda)x^-_{j})}}{e^{f(x)^Tf(\lambda x^++(1-\lambda)x^-_{j})}+N  g_0(x,\{x_{i_2 j}^-\}^N_{i_2})}.$
Interestingly, the construction of $x'$ herein assembles the mixup~\cite{zhang2018mixup} philosophy in supervised learning. 
Recent work \cite{lee2021imix,verma2021towards} have also considered augment the dataset by including synthetic data point and build domain-agnostic contrastive learning strategies, however, their loss is different from this work because they apply mixup on the data points $x$ while we use mixup to produce diverse postivie pairs. 


\begin{table*}[t!]
\centering
\caption{The relationship between IntNaCl framework and the literature: existing works are special cases of $\mathcal{L}_{\text{IntNaCl}}$ } 
\label{tbl:summary_1}
\scalebox{0.8}
{
\begin{tabular}{c|l|cccc|c|cc}
\hlineB{3}
& \multirow{2}{*}{$\mathcal{L}_{\text{IntNaCl}}$} & \multicolumn{4}{c|}{$\mathcal{L}_{\text{NaCl}}(G^1, \M, \lambda)$} & \multirow{2}{*}{$\alpha$} & \multicolumn{2}{c}{$\mathcal{L}_{\text{Robust}}(G^2, w))$}\\
& & $\mathcal{L}_{\text{NaCl}}$ & $G^1$ & $\M$ & $\lambda$ &  & $G^2$ & $w$ \\\hline
& $\mathcal{L}_{\text{SimCLR}}$~\cite{chen2020simple} &   $\mathcal{L}_{\text{NCA}}$/$\mathcal{L}_{\text{MIXNCA}}$ & $g_0$ & 1 & - & 0 & - & -\\
Existing & $\mathcal{L}_{\text{Debiased}}$~\cite{chuang2020debiased} &   $\mathcal{L}_{\text{NCA}}$/$\mathcal{L}_{\text{MIXNCA}}$ & $g_1$ & 1 & - & 0 & - & -\\
Work & $\mathcal{L}_{\text{Debiased+HardNeg}}$~\cite{robinson2021contrastive} &   $\mathcal{L}_{\text{NCA}}$/$\mathcal{L}_{\text{MIXNCA}}$ & $g_2$ & 1 & - & 0 & - & -\\
& $\mathcal{L}_{\text{Adv}}$~\cite{ho2020contrastive} &   $\mathcal{L}_{\text{NCA}}$/$\mathcal{L}_{\text{MIXNCA}}$ & $g_0$ & 1 & - & 1 & $g_0$ & 1\\\hline
& $\mathcal{L}_{\text{IntCl}}$ in Fig.~\ref{fig:cifars} &   $\mathcal{L}_{\text{NCA}}$/$\mathcal{L}_{\text{MIXNCA}}$ & $g_2$ & 1 & - & 1 & $g_2$ & $\hat{w}(x)$\\
& $\mathcal{L}_{\text{IntNaCl}}$ in Fig.~\ref{fig:cifars} & $\mathcal{L}_{\text{MIXNCA}}$ & $g_2$ & 5 & 0.5 & 1 & $g_2$ & $\hat{w}(x)$\\
Our & $\mathcal{L}_{\text{IntNaCl}}$ in Tab.~\ref{tab:ncacl} & $\mathcal{L}_{\text{NCA}}$/$\mathcal{L}_{\text{MIXNCA}}$ & $g_0$/$g_2$ & 1-5 & 0.5/0.9 & 0 & - & -\\
Method & $\mathcal{L}_{\text{IntNaCl}}$ in Tab.~\ref{tab:lam_pos_3} & $\mathcal{L}_{\text{NCA}}$/$\mathcal{L}_{\text{MIXNCA}}$ & $g_2$ & 1-5 & 0.5/0.7/0.9 & 1 & $g_2$ & $\hat{w}(x)$\\
& $\mathcal{L}_{\text{IntNaCl}}$ in Fig.~\ref{fig:MIXNCA_lam} & $\mathcal{L}_{\text{MIXNCA}}$ & $g_0$/$g_2$ & 1-5 & 0.5-0.9 & 0 & - & -\\
& $\mathcal{L}_{\text{IntNaCl}}$ in Tab.~\ref{tab:tinyimagenet} & $\mathcal{L}_{\text{NCA}}$/$\mathcal{L}_{\text{MIXNCA}}$ & $g_0$/$g_2$ & 1/2/5 & 0.5/0.9 & 0/1 & -/$g_2$ & -/$\hat{w}(x)$/1\\\hlineB{3}
\end{tabular}
}
\end{table*}

\subsection{Integrated contrastive learning framework}
\label{sec:tricks}
Building on top of NaCl, we can propose a useful framework \textbf{IntNaCl} that not only generalizes existing methods but also achieves good accuracy and robustness simultaneously. 
Before we introduce IntNaCl, we give an intermediate integrated loss as \textbf{IntCl}, which consists of two components -- a standard loss and a robustness-promoting loss. 

Motivated by $\mathcal{L}_{\text{Adv}}$~\cite{ho2020contrastive}, we consider a robust-promoting loss defined by
\begin{align*}
    \mathcal{L}_{\text{Robust}}(G, w) \!\! := \!\E \left[
    -  \log\frac{e^{f(x)^Tf(x^{\text{adv}})}}{e^{f(x)^Tf(x^{\text{adv}})}+N  G(x,\cdot)}  w(x)
    \right],
\end{align*}
where $G$ can be chose from \{$g_0,g_1,g_2$\}, and $w(x)$ facilitates goal-specific weighting schemes. Note that $w(x)$ can be a general function and $\mathcal{L}_{\text{Adv}}$~\cite{ho2020contrastive} is a special case when $w(x) = 1$. 

\paragraph{Adversarial weighting.} Weighting sample loss based on their margins has been proven to be effective in the adversarial training under supervised settings~\cite{zeng2020adversarial}. Specifically, it is argued that training points that are closer to the decision boundaries should be given more weight in the supervised loss. 
While the margin of a sample is underdefined in unsupervised settings, we can give our weighting function as the value of the contrastive loss
$\hat{w}(x) :=
-\log\frac{e^{f(x)^Tf(x^+)}}{e^{f(x)^Tf(x^+)}+N  G(x,\cdot)}$.
Using this, we see that samples that are originally hard to be distinguished from other samples (i.e. small probability) are now assigned with bigger weights. Below, we propose a new integrated framework to involve the robustness term $\mathcal{L}_{\text{Robust}}(G, w)$ which can greatly help on promoting robustness in contrastive learning. In particular, we show that many existing contrastive learning losses are special cases of our proposed framework. 

\paragraph{IntCl.} For IntCl, the standard loss can be existing contrastive learning losses~\cite{chen2020simple,chuang2020debiased,robinson2021contrastive}, which correspond to a form of 
\begin{align*}
    & (\text{IntCL loss }\mathcal{L}_{\text{IntCL}}) \nonumber\\
    & \E 
    \left[-\log\frac{e^{f(x)^Tf(x^+)}}{e^{f(x)^Tf(x^+)}+N G^1(x,\cdot)} \right] +\alpha\mathcal{L}_{\text{Robust}} (G^2, w),
\end{align*}
with $G^1$ and $G^2$ being $g_0$, $g_1$, and $g_2$.   Unless otherwise specified, we use the adversarial weighting scheme introduced above throughout our experiments. Notice that $\mathcal{L}_{\text{IntCL}}$ reduces to $\mathcal{L}_{\text{Adv}}$ when $G^1=G^2=g_0$ and $w(x)\equiv 1$.

\paragraph{\textbf{IntNaCl}.} To design a generic loss that accounts for robust accuracy while keeping clean accuracy, we utilize $\mathcal{L}_{\text{NaCl}}$ developed in Section~\ref{subsec:NaCl} to strength the standard loss in $\mathcal{L}_{\text{IntCL}}$. We call this ultimate framework \textbf{Int}egrated \textbf{N}eighborhood \textbf{a}nalysis \textbf{C}ontrastive \textbf{l}oss (\textbf{IntNaCl}), which is given by
\begin{align}
\label{eqn:IntNCACL}
    \mathcal{L}_{\text{IntNaCl}} := \mathcal{L}_{\text{NaCl}}(G^1, \M, \lambda) +\alpha\mathcal{L}_{\text{Robust}} (G^2, w),
\end{align}
where $\mathcal{L}_{\text{NaCl}}(G^1, \M, \lambda)$ can be chose from \{$\mathcal{L}_{\text{NCA}}(G^1, \M)$, $\mathcal{L}_{\text{MIXNCA}}(G^1, \M, \lambda)$\}.
We remark that as $\mathcal{L}_{\text{NCA}}$ and $\mathcal{L}_{\text{MIXNCA}}$ all reduce to one same form when $\M=1$, the $\mathcal{L}_{\text{IntNaCl}}$ under $M=1$ is exactly $\mathcal{L}_{\text{IntCl}}$. 
This general framework includes many of the existing works as special cases and we summarize these relationships in Table~\ref{tbl:summary_1}.

\section{Experimental Results}
\label{sec:exp}

\begin{table*}[t!]
    \centering
    \caption{Performance comparisons of $\mathcal{L}_{\text{NaCl}}$ ($\M\neq 1$) and i) \textit{Left}: SimCLR~\cite{chuang2020debiased} ($\M=1, G^1=g_0$) and ii) \textit{Right}: Debised+HardNeg~\cite{robinson2021contrastive} ($\M=1, G^1=g_2$) when $\alpha=0$. The best performance within each loss type is in boldface (larger is better).}
    \label{tab:ncacl}
    \scalebox{0.8}
    {\begin{tabular}{c|cccc|cccc}
      \hlineB{3}
    \multirow{2}{*}{$\M$} & \multicolumn{4}{c|}{$\alpha=0,~\mathcal{L}_{\text{NaCl}}(G^1, \M, \lambda)=\mathcal{L}_{\text{NCA}}(g_0, \M)$} & \multicolumn{4}{c}{$\alpha=0,~\mathcal{L}_{\text{NaCl}}(G^1, \M, \lambda)=\mathcal{L}_{\text{NCA}}(g_2, \M)$} \\
         &  CIFAR100 Acc. & FGSM Acc. & CIFAR10 Acc. & FGSM Acc. &  CIFAR100 Acc. & FGSM Acc. & CIFAR10 Acc. & FGSM Acc. \\\hline
      1 & 53.69$\pm$0.25 & 25.17$\pm$0.55 & 76.34$\pm$0.28 & 43.50$\pm$0.41 & 56.83$\pm$0.20 & 31.03$\pm$0.41 & 77.24$\pm$0.29 & 48.38$\pm$0.70\\
      2 & 55.72$\pm$0.15 & 27.04$\pm$0.45 & 77.40$\pm$0.14 & 44.58$\pm$0.41 & 57.87$\pm$0.15 & 32.50$\pm$0.48 & 77.43$\pm$0.11 & 48.14$\pm$0.31\\
      3 & 56.67$\pm$0.12 & \textbf{28.41$\pm$0.24} & 77.53$\pm$0.24 & \textbf{45.21$\pm$0.89} & 58.42$\pm$0.23 & \textbf{33.19$\pm$0.60} & 77.41$\pm$0.17 & 48.09$\pm$0.93\\
      4 & 57.09$\pm$0.26 & 28.20$\pm$0.81 & 77.75$\pm$0.22 & 45.13$\pm$0.44 & \textbf{58.86$\pm$0.18} & 32.65$\pm$1.07 & 77.46$\pm$0.29 & \textbf{48.43$\pm$0.94}\\
      5 & \textbf{57.32$\pm$0.17} & 28.33$\pm$0.59 & \textbf{77.93$\pm$0.40} & 44.46$\pm$0.53 & 58.81$\pm$0.21 & 32.86$\pm$0.47 & \textbf{77.58$\pm$0.23} & 48.30$\pm$0.39\\\hline
      & \multicolumn{4}{c|}{$\alpha=0,~\mathcal{L}_{\text{NaCl}}(G^1, \M, \lambda)=\mathcal{L}_{\text{MIXNCA}}(g_0, \M, 0.9)$} & \multicolumn{4}{c}{$\alpha=0,~\mathcal{L}_{\text{NaCl}}(G^1, \M, \lambda)=\mathcal{L}_{\text{MIXNCA}}(g_2, \M, 0.5)$}\\\hline
      1 & 53.69$\pm$0.25 & 25.17$\pm$0.55 & 76.34$\pm$0.28 & 43.50$\pm$0.41 & 56.83$\pm$0.20 & 31.03$\pm$0.41 & 77.24$\pm$0.29 & 48.38$\pm$0.70\\
      2 & 56.20$\pm$0.33 & 30.95$\pm$0.36 & 76.96$\pm$0.15 & \textbf{48.85$\pm$0.75} & 59.41$\pm$0.19 & \textbf{32.22$\pm$0.35} & 79.36$\pm$0.65 & 48.86$\pm$0.34\\
      3 & 56.41$\pm$0.13 & \textbf{30.98$\pm$0.90} & 77.10$\pm$0.21 & 48.76$\pm$0.63 & 59.81$\pm$0.25 & 32.04$\pm$0.67 & 79.41$\pm$0.17 & 48.91$\pm$0.81\\
      4 & 56.00$\pm$0.42 & 29.90$\pm$0.63 & \textbf{77.11$\pm$0.40} & 48.16$\pm$0.40 & 59.75$\pm$0.33 & 32.03$\pm$0.34 & 79.42$\pm$0.18 & \textbf{49.05$\pm$0.71}\\
      5 & \textbf{56.63$\pm$0.31} & 30.58$\pm$0.52 & 77.04$\pm$0.19 & 47.96$\pm$0.46 & \textbf{59.85$\pm$0.30} & 32.06$\pm$0.72 & \textbf{79.45$\pm$0.20} & 48.32$\pm$0.70\\
      \hlineB{3}
    \end{tabular}
    }

    \centering
    \caption{Performance comparisons of $\mathcal{L}_{\text{IntNaCl}}$ ($\M\neq 1$) and $\mathcal{L}_{\text{IntCL}}$ ($\M= 1$) when $\alpha=1, G^1=G^2=g_2, w = \hat{w}(x)$. The best performance within each loss type is in boldface (larger is better).}
    \label{tab:lam_pos_3}
    \scalebox{0.8}
    {\begin{tabular}{c|cccc|cccc}
    \hlineB{3}
    \multirow{2}{*}{$\M$} & \multicolumn{4}{c|}{$\alpha\neq0,~\mathcal{L}_{\text{NaCl}}(G^1, \M, \lambda)=\mathcal{L}_{\text{NCA}}(g_2, \M)$} & \multicolumn{4}{c}{$\alpha\neq0,~\mathcal{L}_{\text{NaCl}}(G^1, \M, \lambda)=\mathcal{L}_{\text{MIXNCA}}(g_2, \M,0.5)$} \\
      &  CIFAR100 Acc. & FGSM Acc. & CIFAR10 Acc. & \multicolumn{1}{c|}{FGSM Acc.} &  CIFAR100 Acc. & FGSM Acc. & CIFAR10 Acc. & FGSM Acc. \\\hline
      1 & 56.22$\pm$0.15 & 40.05$\pm$0.67 & 76.39$\pm$0.10 & \multicolumn{1}{c|}{\textbf{59.33$\pm$0.94}} & 56.22$\pm$0.15 & 40.05$\pm$0.67 & 76.39$\pm$0.10 & \textbf{59.33$\pm$0.94}\\
      2 &  56.71$\pm$0.11 & 39.80$\pm$0.57 & 76.55$\pm$0.27 & 58.44$\pm$0.31 & 58.97$\pm$0.19 & 40.25$\pm$0.52 & 78.61$\pm$0.20 & 58.41$\pm$0.59\\
      3 &  57.13$\pm$0.26 & 40.53$\pm$0.29 & \textbf{76.67$\pm$0.22} & 58.47$\pm$0.31 & 59.26$\pm$0.18 & 40.96$\pm$0.58 & \textbf{78.83$\pm$0.22} & 59.20$\pm$1.25\\
      4 &  57.06$\pm$0.19 & 40.85$\pm$0.31 & 76.34$\pm$0.22 & 58.91$\pm$0.62 & 59.32$\pm$0.21 & 40.82$\pm$0.54 & 78.83$\pm$0.27 & 59.03$\pm$0.52\\
      5 &  \textbf{57.46$\pm$0.04} & \textbf{41.00$\pm$0.86} & 76.60$\pm$0.37 & 57.98$\pm$0.47 & \textbf{59.43$\pm$0.23} & \textbf{41.01$\pm$0.34} & 78.80$\pm$0.21 & \textbf{59.51$\pm$0.93}\\\hline
      & \multicolumn{4}{c|}{$\alpha\neq0,~\mathcal{L}_{\text{NaCl}}(G^1, \M, \lambda)=\mathcal{L}_{\text{MIXNCA}}(g_2, \M,0.7)$} & \multicolumn{4}{c}{$\alpha\neq0,~\mathcal{L}_{\text{NaCl}}(G^1, \M, \lambda)=\mathcal{L}_{\text{MIXNCA}}(g_2, \M,0.9)$}\\\hline
      1 & 56.22$\pm$0.15 & 40.05$\pm$0.67 & 76.39$\pm$0.10 & 59.33$\pm$0.94 & 56.22$\pm$0.15 & 40.05$\pm$0.67 & 76.39$\pm$0.10 & 59.33$\pm$0.94\\
      2 & 58.00$\pm$0.18 & 40.35$\pm$0.34 & 77.73$\pm$0.24 & 59.40$\pm$1.27 & 56.54$\pm$0.33 & 40.85$\pm$0.13 & 76.81$\pm$0.22 & 60.40$\pm$0.46\\
      3 & 58.23$\pm$0.18 & 40.94$\pm$0.75 & 77.91$\pm$0.25 & \textbf{59.57$\pm$0.81} & 56.69$\pm$0.11 & 41.23$\pm$0.66 & \textbf{76.98$\pm$0.22} & 60.13$\pm$0.56\\
      4 & 58.20$\pm$0.25 & 40.95$\pm$0.45 & 77.89$\pm$0.20 & 59.49$\pm$0.49 & 56.43$\pm$0.26 & \textbf{41.56$\pm$0.56} & 76.97$\pm$0.20 & \textbf{61.21$\pm$0.49}\\
      5 & \textbf{58.37$\pm$0.14} & \textbf{41.15$\pm$0.48} & \textbf{78.27$\pm$0.26} & 59.17$\pm$0.94 & \textbf{56.86$\pm$0.11} & 41.09$\pm$0.31 & 76.91$\pm$0.21 & 60.09$\pm$0.39\\\hlineB{3}
    \end{tabular}}
\end{table*}

\paragraph{Implementation details.}
All the proposed methods are implemented based on open source repositories provided in the literature~\cite{chen2020simple,ho2020contrastive,robinson2021contrastive}. Five benchmarking contrastive losses are considered as baselines that include: $\mathcal{L}_{\text{SimCLR}}$~\cite{chen2020simple}, $\mathcal{L}_{\text{Debiased}}$~\cite{chuang2020debiased}, $\mathcal{L}_{\text{Debiased+HardNeg}}$~\cite{robinson2021contrastive}, $\mathcal{L}_{\text{Adv}}$~\cite{ho2020contrastive} (i.e. Equation \eqref{eqn:simclr},~Equation \eqref{eqn:debiased},~Equation \eqref{eqn:harddebiased},~Equation \eqref{eqn:adversarial}). We train representations on resnet18 and include MLP projection heads~\cite{chen2020simple}. A batch size of 256 is used for all CIFAR~\cite{krizhevsky2009learning} experiments and a batch size of 128 is used for all tinyImagenet experiments. Unless otherwise specified, the representation network is trained for 100 epochs. We run five independent trials for each of the experiments and report the mean and standard deviation in the entries. 
We implement the proposed framework using PyTorch to enable the use of an NVIDIA GeForce RTX 2080 Super GPU and four NVIDIA Tesla V100 GPUs. 

\paragraph{Evaluation protocol.} We follow the standard evaluation protocal to report three major properties of representation learning methods: standard discriminative power, transferability, and adversarial robustness. 
To evaluate the standard discriminative power, we train representation networks on CIFAR100/tinyImagenet, freeze the network, and fine-tune a fully-connected layer that maps representations to outputs on CIFAR100/tinyImagenet, which is consistent with the standard linear evaluation protocol in the literature~\cite{chen2020simple,chuang2020debiased,Grill2020Bootstrap,ho2020contrastive,khosla2020supervised,Tian2020what,robinson2021contrastive,saunshi2019theoretical,kim2020adversarial,haochen2021provable}. 
To evaluate the transferability, we use the representation networks trained on CIFAR100, and only fine-tune a fully-connected layer that maps representations to outputs on CIFAR10.
All the adversarial robustness evaluations are based on the implementation provided by~\cite{Wong2020Fast}. We supplement more FGSM and PGD attack results in the appendix.

\paragraph{Experiment outline.} 
Since the performance of the integrated method $\mathcal{L}_{\text{IntNaCl}}$ is attributed to multiple components in the formulation (Equation~\ref{eqn:IntNCACL}), we do ablation studies in the following sections to study their effectiveness individually. In Section~\ref{subsec:effect_nacl}, we evaluate the effect of $\mathcal{L}_{\text{NaCl}}$; in Section~\ref{subsec:effect_int}, we evaluate the effect of $\mathcal{L}_{\text{Robust}}$; in Section~\ref{subsec:effect_small}, we evaluate the effect of $M$, $\lambda$, and $w$.

\begin{table*}[t!]
    \centering
    \caption{Performance comparisons of $\mathcal{L}_{\text{NaCl}}$ and $\mathcal{L}_{\text{IntNaCl}}$ with baselines on TinyImagenet. The best performance within each loss type is in boldface (larger is better).}
    \label{tab:tinyimagenet}
    \scalebox{0.8}
    {
    \begin{tabular}{c|cc|cc}
      \hlineB{3}
    $\alpha=0$ & \multicolumn{2}{c|}{$\mathcal{L}_{\text{NaCl}}(G^1, \M, \lambda)=\mathcal{L}_{\text{NCA}}(g_0, \M)$} & \multicolumn{2}{c}{$\mathcal{L}_{\text{NaCl}}(G^1, \M, \lambda)=\mathcal{L}_{\text{NCA}}(g_2, \M)$} \\
    $\M$     &  TinyImagenet Acc. & FGSM Acc. & TinyImagenet Acc. & FGSM Acc. \\\hline
      1 & 39.66$\pm$0.15 & 24.80$\pm$0.07 & 41.26$\pm$0.14 & 27.34$\pm$0.77\\
      2 & \textbf{40.71$\pm$0.26} & \textbf{26.29$\pm$0.51} & \textbf{41.99$\pm$0.23} & \textbf{28.14$\pm$0.13}\\\hline
      & \multicolumn{2}{c|}{$\mathcal{L}_{\text{NaCl}}(G^1, \M, \lambda)=\mathcal{L}_{\text{MIXNCA}}(g_0, \M, 0.9)$} & \multicolumn{2}{c}{$\mathcal{L}_{\text{NaCl}}(G^1, \M, \lambda)=\mathcal{L}_{\text{MIXNCA}}(g_2, \M, 0.5)$}\\\hline
      1 & 39.66$\pm$0.15 & 24.80$\pm$0.07 & 41.26$\pm$0.14 & 27.34$\pm$0.77\\
      2 & \textbf{40.23$\pm$0.37} & \textbf{26.47$\pm$0.24} & \textbf{43.91$\pm$0.20} & \textbf{28.29$\pm$0.33}\\
      \hlineB{3}
    $\alpha=1$  & \multicolumn{2}{c|}{$\mathcal{L}_{\text{NaCl}}(G^1, \M, \lambda)=\mathcal{L}_{\text{MIXNCA}}(g_2, \M,0.5)$} & \multicolumn{2}{c}{$\mathcal{L}_{\text{NaCl}}(G^1, \M, \lambda)=\mathcal{L}_{\text{MIXNCA}}(g_2, \M,0.5)$}\\
      & \multicolumn{2}{c|}{$\mathcal{L}_{\text{Robust}}(G^2, w))=\mathcal{L}_{\text{Robust}}(g_2, \hat{w}(x)))$} & \multicolumn{2}{c}{$\mathcal{L}_{\text{Robust}}(G^2, w))=\mathcal{L}_{\text{Robust}}(g_2, 1))$}\\\hline
      1 & 42.56$\pm$0.13 & 31.18$\pm$0.51 & 42.24$\pm$0.14 & 31.55$\pm$0.38\\
      2 & 44.69$\pm$0.20 & \textbf{32.65$\pm$0.52} & 44.37$\pm$0.08 & 32.20$\pm$0.23\\
      5 & \textbf{45.31$\pm$0.22} & 32.43$\pm$0.33 & \textbf{44.77$\pm$0.11} & \textbf{32.47$\pm$0.42}\\\hlineB{3}
    \end{tabular}
    }
\end{table*}

\subsection{The effect of $\mathcal{L}_{\text{NaCl}}$}
\label{subsec:effect_nacl}
By evaluating the effect of $\mathcal{L}_{\text{NaCl}}$, we want to evaluate the performance difference of our framework $\mathcal{L}_{\text{IntNaCl}}$ when $M\geq 1$ and $M=1$. In order to see that, we consider 2 cases: (1) set $\alpha=0$ in Equation~\eqref{eqn:IntNCACL} and compare $\mathcal{L}_{\text{NaCl}}(G^1, \M\neq 1, \lambda)$ with existing work $\mathcal{L}_{\text{NaCl}}(G^1, \M=1, \lambda)$, or (2) set $\alpha=1$ and compare $\mathcal{L}_{\text{IntNaCl}}$ and $\mathcal{L}_{\text{IntCl}}$. 

\textbf{Case (1) $\alpha=0$.} In Table \ref{tab:ncacl}, after setting $\alpha=0$, we experiment with $G^1=g_0,g_2$. By referring to Table~\ref{tbl:summary_1}, our baseline becomes exactly SimCLR~\cite{chen2020simple} when $G^1=g_0$, and becomes Debiased+HardNeg~\cite{robinson2021contrastive} when $G^1=g_2$. 
From Table~\ref{tab:ncacl}, one can see that when $M\neq 1$, $\mathcal{L}_{\text{NCA}}$ and $\mathcal{L}_{\text{MIXNCA}}$ can both improve upon the baselines($M=1$) in all metrics (standard/robust/transfer accuracy). When $G^1=g_0$, $\mathcal{L}_{\text{NCA}}$'s improvement over SimCLR also exemplifies our Theorem~\ref{thm:generalization}.
Due to page limits, we only select one $\lambda$ when $\mathcal{L}_{\text{NaCl}}=\mathcal{L}_{\text{MIXNCA}}$ and report results together with the results of $\mathcal{L}_{\text{NaCl}}=\mathcal{L}_{\text{NCA}}$. Full tables can be found in the appendix~\ref{sec:complete}. We further verify the performance on TinyImagent and give results in Table~\ref{tab:tinyimagenet}. Notice that now when $G^1=g_0$, we are using a batch size of $N=128$ for 200-class TinyImagent task. Therefore, the requirement of $N\geq K-1$ in Theorem~\ref{thm:generalization} is not fulfilled. However, we can still see improvements when going from $M=1$ to $M=2$.

\textbf{Case (2) $\alpha=1$.} In Table~\ref{tab:lam_pos_3}, after setting $\alpha=1$, we experiment with $G^1=G^2=g_2$ since $g_2$ generally yields better performance in Table \ref{tab:ncacl}.
When $\mathcal{L}_{\text{NaCl}}(G^1, \M, \lambda) = \mathcal{L}_{\text{MIXNCA}}(g_2, \M, \lambda)$, we give the results for $\lambda=0.5,0.7,0.9$ to show an interesting effect: while $\mathcal{L}_{\text{MIXNCA}}(g_2, \M, \lambda=0.5)$ benefits a lot going from $\M=1$ to $\M=5$ (standard accuracy increases from 56.22\% to 59.43\%), the improvement is comparatively smaller with $\mathcal{L}_{\text{MIXNCA}}(g_2, \M, 0.9)$ (standard accuracy increases from 56.22\% to 56.86\%). 
In Figure~\ref{fig:cifars}, we plot the robust accuracy defined under FGSM attacks~\cite{goodfellow2014explaining} along the y-axis. Ideally, one desires a representation network that pushes the performance to the upper-right corner in the 2D accuracy grid (standard-robust accuracy plot).
We highlight the results of $\mathcal{L}_{\text{IntNaCl}}$ and $\mathcal{L}_{\text{IntCL}}$ in circles, through which we see that while $\mathcal{L}_{\text{IntCL}}$ can already train representations that are decently robust without sacrificing the standard accuracy on CIFAR100, the standard accuracy on CIFAR10 is inferior to some baselines (HardNeg and Debiased+HardNeg). Comparatively, $\mathcal{L}_{\text{IntNaCl}}$ demonstrates high transfer standard accuracy and wins over the baselines by a large margin on both datasets, proving the ability of learning representation networks that also transfer robustness property. For TinyImagent, we only show the results when $\mathcal{L}_{\text{NaCl}}(G^1, \M, \lambda) = \mathcal{L}_{\text{MIXNCA}}(g_2, \M, 0.5)$ since $g_2$ generally achieves higher accuracy and combines well with $\mathcal{L}_{\text{MIXNCA}}$. Importantly, with the help of $\mathcal{L}_{\text{NaCl}}$ module, the performance can be boosted from 42.56\% to 45.31\% while maintaining good robust accuracy 32.43\%.

\subsection{The effect of $\mathcal{L}_{\text{Robust}}$}
\label{subsec:effect_int}
By evaluating the effect of $\mathcal{L}_{\text{Robust}}$, we want to see the performance difference of our framework $\mathcal{L}_{\text{IntNaCl}}$ when $\alpha\neq 0$ and $\alpha=0$. Therefore, we consider 2 cases: (1) set $\M=1$ in Equation~\eqref{eqn:IntNCACL} and compare $\mathcal{L}_{\text{IntCl}}$ with existing work $\mathcal{L}_{\text{NaCl}}(G^1, \M=1, \lambda)$, or (2) set $\M\neq 1$ and compare $\mathcal{L}_{\text{IntNaCl}}$ and $\mathcal{L}_{\text{NaCl}}(G^1, \M\neq 1, \lambda)$. 

\textbf{Case (1) $\M=1$.} Notice that $\mathcal{L}_{\text{IntCL}}$ differs from standard contrastive losses by including the term $\mathcal{L}_{\text{Robust}}$. Therefore, one can easily evaluate the effect of $\mathcal{L}_{\text{Robust}}$ by inspecting the performance difference between $\mathcal{L}_{\text{IntCl}}$ and the baselines in Figure~\ref{fig:cifars}.
Specifically, we let $G^1=g_2$ for $\mathcal{L}_{\text{IntCl}}$ in Figure~\ref{fig:cifars}, hence a direct baseline is Debiased+HardNeg. By adding a robustness-promoting term, the robust accuracy can be boosted from 31.03\% to 40.05\% and transfer robust accuracy from 48.38\% to 59.33\%, which is a significant improvement.

\textbf{Case (2) $\M\neq 1$.} The effect of $\mathcal{L}_{\text{Robust}}$ is also demonstrated through the robust accuracy ``jump'' from Table~\ref{tab:ncacl} to Table~\ref{tab:lam_pos_3}. For example, we point out that in Table~\ref{tab:ncacl}, $\mathcal{L}_{\text{NaCl}}(G^1, \M, \lambda)=\mathcal{L}_{\text{NCA}}(g_2, 3)$ gives the maximum robust accuracy of 33.19\%, while the robust accuracy obtained with the same $\mathcal{L}_{\text{NaCl}}(G^1, \M, \lambda)=\mathcal{L}_{\text{NCA}}(g_2, 3)$ and additional $\mathcal{L}_{\text{Robust}}$ increases to 40.53\% in Table~\ref{tab:lam_pos_3}. The robust accuracy boost on TinyImagent with the help of $\mathcal{L}_{\text{Robust}}$ is also visible: when $\mathcal{L}_{\text{NaCl}}(G^1, \M, \lambda)=\mathcal{L}_{\text{MIXNCA}}(g_2, 2, 0.5)$, the robust accuracy increases from 28.29\% to 32.65\%. 


\subsection{The effect of $\M$, $\lambda$, and $w(x)$}
\label{subsec:effect_small}
To evaluate the effect of $\M$, we can see from Table~\ref{tab:ncacl} and Table~\ref{tab:lam_pos_3} that the performance is generally increasing as $\M$ increases. However, this effect seems to be less visible for robust accuracy and transfer robust accuracy.
To evaluate the effect of $\lambda$, we include in Figure~\ref{fig:MIXNCA_lam} the standard and robust accuracy on CIFAR100 and CIFAR10 as functions of $\lambda$. Intriguingly, we see that the accuracy curves mainly show trends of increasing in Figure~\ref{fig:MIXNCA_lam}(a). Comparatively, the standard accuracy on CIFAR100 and CIFAR10 shows trends of decreasing in Figure~\ref{fig:MIXNCA_lam}(b). One possible explanation is by the original baselines' room for improvement. Since Debiased+HardNeg is a much stronger baseline than SimCLR, it is closer to the robustness-accuracy trade-off. However, we note that the overall performance of NaCl on Debiased+HardNeg is still better than NaCl on SimCLR regardless of the robustness-accuracy trade-off. In the last row of Table~\ref{tab:tinyimagenet}, we list the results when $\mathcal{L}_{\text{NaCl}}(G^1, \M, \lambda)=\mathcal{L}_{\text{MIXNCA}}(g_2, \M,0.5)$ but different $\mathcal{L}_{\text{Robust}}(G^2, w))$. Specifically, on the left we show the case when $w=\hat{w}(x)$ and on the right we show the case when $w=1$. One can then see that by using a goal-specific weighting scheme, the performance can be further boosted.

\section{Conclusion}
\label{sec:conclusion}
In this paper, we discover the relationship between contrastive loss and Neighborhood Component Analysis (NCA), which motivates us to generalize the existing contrastive loss to a set of Neighborhood analysis Contrastive losses (NaCl). We further propose a generic and integrated contrastive learning framework (IntNaCl) based on NaCl, which learns representations that score high in both standard accuracy and adversarial accuracy in downstream tasks. With the integrated framework, we can boost the standard accuracy by 6\% and the robust accuracy by 17\%. 






\clearpage
\bibliography{icml2022}
\bibliographystyle{icml_template/icml2022}

\newpage
\appendix
\onecolumn
\renewcommand{\thetable}{S\arabic{table}}
\renewcommand{\thefigure}{S\arabic{figure}}
\setcounter{figure}{0}
\setcounter{table}{0}

\section{Derivation from Equation \eqref{eqn:NCA} to \eqref{eqn:end}}
Since a generalization of contrastive learning loss in Equation~\eqref{eqn:simclr} can be given by assuming (a) positive pairs belong to the same class and (b) the transformation $Ax$ is instead parametrized by a general function $\frac{f(x)}{\sqrt{2}} := \frac{h(x)}{\sqrt{2}\norm{h(x)}}$, where $h$ is a neural network, Equation \eqref{eqn:NCA} becomes Equation \eqref{eqn:start}:
\begin{align}
    \min_f \sum_{i=1}^n&~ - \log\left(\sum_{j = 1}^\M\frac{e^{-\frac{1}{2}\norm{f(x_i) - f(x_{ij}^+)}^2}}{ \sum_{k\neq i}e^{-\frac{1}{2}\norm{f(x_i) - f(x_k)}^2}}\right). \label{eqn:start}\tag{S1}
\end{align}
Then we can prove
\begin{align}
    &\argmin_f \sum_{i=1}^n~ - \log\left(\sum_{j = 1}^\M\frac{e^{-\frac{1}{2}\norm{f(x_i) - f(x_{ij}^+)}^2}}{ \sum_{k\neq i}e^{-\frac{1}{2}\norm{f(x_i) - f(x_k)}^2}}\right) \nonumber\\
    =&\argmin_f \sum_{i=1}^n~ - \log\left(\sum_{j = 1}^\M\frac{e^{f(x_i)^Tf(x_{ij}^+)-\frac{1}{2}\norm{f(x_i)}^2-\frac{1}{2}\norm{f(x_{ij}^+)}^2}}{ \sum_{k\neq i}e^{f(x_i)^Tf(x_k)-\frac{1}{2}\norm{f(x_i)}^2-\frac{1}{2}\norm{f(x_k)}^2}}\right) \label{eqn:2}\tag{S2}\\
    =&\argmin_f \sum_{i=1}^n~ - \log\left(\sum_{j = 1}^\M\frac{e^{f(x_i)^Tf(x_{ij}^+)-1}}{ \sum_{k\neq i}e^{f(x_i)^Tf(x_k)-1}}\right) \label{eqn:3}\tag{S3}\\
    =&\argmin_f \sum_{i=1}^n~ - \log\left(\frac{\sum\limits_{j=1}\limits^\M e^{f(x_i)^Tf(x_{ij}^+)}}{ \sum_{k\neq i}e^{f(x_i)^Tf(x_k)}}\right) \nonumber\\
    =&\argmin_f \sum_{i=1}^n~ - \log\left(\frac{\sum\limits_{j=1}\limits^\M e^{f(x_i)^Tf(x_{ij}^+)}}{ \sum_{k\neq i, x_k\in\{x_{ij}^+\}}e^{f(x_i)^Tf(x_k)}+\sum_{k\neq i, x_k\notin\{x_{ij}^+\}}e^{f(x_i)^Tf(x_k)}}\right) \label{eqn:5}\tag{S4}\\
    =&\argmin_f \mathbb{E}_{x\sim \D}\left[- \log\left(\frac{\sum\limits_{j=1}\limits^\M e^{f(x)^Tf(x_j^+)}}{\sum\limits_{j=1}\limits^\M e^{f(x)^Tf(x_j^+)} +\sum\limits_{i=1}\limits^Ne^{f(x)^Tf(x_i^-)}}\right)\right]\label{eqn:6}\tag{S5}\\
    =&\argmin_f \mathbb{E}_{x\sim \D}\left[- \log\left(\frac{\sum\limits_{j=1}\limits^\M e^{f(x)^Tf(x_j^+)}}{\sum\limits_{j=1}\limits^\M e^{f(x)^Tf(x_j^+)} +N  g_0(x,\{x_i^-\}^N)}\right)\right], \nonumber
\end{align}
where we go from Equation \eqref{eqn:2} to Equation \eqref{eqn:3} based on the fact that $\norm{f(x)}=1$, and from Equation \eqref{eqn:5} to Equation \eqref{eqn:6} assuming that set $\{x_k:k\neq i\}=\{x_j^+:1\leq j\leq \M\}\cup\{x_i^-:1\leq i\leq N\}$. 

\newpage
\section{Generalization Bounds}

We extend the theorems from \cite{chuang2020debiased} to get results for $\mathcal{L}_{\text{NCA}}$. The results we have here apply to $G=g_0$ and $g_1$. The case when $G=g_2$, $\mathcal{L}_{\text{MIXNCA}}$, and $\mathcal{L}_{\text{IntNaCl}}$ are left as future work.

\subsection{Bridging the empirical estimator and asymptotic objective}
We introduce an intermediate unbiased loss in order to extend our results. Let $h(x,y) = e^{f(x)^\top f(y)}$, then the unbiased loss with multiple positive pairs is given as
\[\widetilde{L}^{\M,N}_{\text{Unbiased}}(f) = \mathbb{E}_{\substack{x \sim p \\ x_i^+ \sim p_x^+ }} \left [\log \frac{ \sum_{i=1}^\M h(x,x_i^+)}{ \sum_{i=1}^\M h(x,x_i^+) +  \M\cdot N \cdot \mathbb{E}_{x^- \sim p_x^-}  h(x,x^-)} \right] \]
Then we can define a debiased loss by 
\[L^{\M,N,n,m}_{\text{Debiased}}(f) = \mathbb{E}_{\substack{x \sim p \\ x_i^+ \sim p_x^+ \\
u_i \sim p; v_i \sim p_x^+}} \left[\log \frac{ \sum_{i=1}^\M h(x,x_i^+)}{ \sum_{i=1}^\M h(x,x_i^+) + \M\cdot N \cdot G(x, \{u_i\}_{i=1}^n, \{v_i\}_{i=1}^m)}\right]. \]
\begin{theorem}\label{thm: unbiased_comparision}
For any embedding $f$ and finite $N$ and $M$, we have
\begin{align}
    \left | \widetilde L_{\textnormal{Unbiased}}^{\M,N}(f) - L_{\substack{\textnormal{Debiased}}}^{\M,N,n,m}(f) \right |
    \leq \frac{e^{3/2}}{\tau^-}\sqrt{\frac{\pi}{2n}} + \frac{e^{3/2}\tau^+}{\tau^-}\sqrt{\frac{\pi}{2m}}. \nonumber 
\end{align}
\end{theorem}
The proof of \ref{thm: unbiased_comparision} is the same as the proof of Theorem 3 in \cite{chuang2020debiased} with the help of the following slightly modified version of Lemma A.2 in \cite{chuang2020debiased}. Now if we let
\begin{align*}
    \Delta = \bigg | -\log \frac{ \sum_{i=1}^\M h(x,x_i^+)}{ \sum_{i=1}^\M h(x,x_i^+) +  \M\cdot N \cdot G(x, \{u_i\}_{i=1}^n, \{v_i\}_{i=1}^m)} + \log \frac{ \sum_{i=1}^\M h(x,x_i^+)}{ \sum_{i=1}^\M h(x,x_i^+) +  \M\cdot N \cdot \mathbb{E}_{x^- \sim p_x^-}  h(x,x^-)} \bigg |,
\end{align*}
where $h(x,\bar{x}) = \exp^{f(x) ^\top f(\bar{x})}$, then one has the following lemma:
\begin{lemma} \label{thm: bound Delta}
Let  $x$ and $x^+$ in $\cal X$ be fixed. Further, let $\{u_i\}_{i=1}^n$ and $ \{v_i\}_{i=1}^m$ be collections of i.i.d. random variables sampled from $p$ and $p_x^+$ respectively. Then for all $\varepsilon > 0$, 
\begin{align*}
\mathbb{P}(\Delta \geq \varepsilon) \leq 2 \exp \left ( - \frac{n \varepsilon^2 (\tau^-)^2}{2e^3} \right ) + 2 \exp \left ( - \frac{m \varepsilon^2 (\tau^-/ \tau^+)^2}{2e^3} \right ) .\end{align*}
\end{lemma}




\begin{proof}[Proof of Lemma \ref{thm: bound Delta}]
We first decompose the probability as
\begin{align*}
    &\mathbb{P}\bigg (\bigg | -\log \frac{ \sum_{i=1}^\M h(x,x_i^+)}{ \sum_{i=1}^\M h(x,x_i^+) +  \M\cdot N \cdot  G(x, \{u_i\}_{i=1}^n, \{v_i\}_{i=1}^m)} + \log \frac{ \sum_{i=1}^\M h(x,x_i^+)}{ \sum_{i=1}^\M h(x,x_i^+) +  \M\cdot N \cdot   \mathbb{E}_{x^- \sim p_x^-}  h(x,x^-)} \bigg | \geq \varepsilon\bigg )  
    \\
   &=\mathbb{P} \bigg ( \bigg | \log \big \{ \sum_{i=1}^\M h(x,x_i^+) + \M\cdot N \cdot  G(x, \{u_i\}_{i=1}^n, \{v_i\}_{i=1}^m) \big \}  - \log \big \{ \sum_{i=1}^\M h(x,x_i^+) + \M\cdot N \cdot  \mathbb{E}_{x^- \sim p_x^-}  h(x,x^-) \big \}  \bigg | \geq \varepsilon \bigg ) 
    \\
   &=\mathbb{P} \bigg (  \log \big \{ \sum_{i=1}^\M h(x,x_i^+) + \M\cdot N \cdot  G(x, \{u_i\}_{i=1}^n, \{v_i\}_{i=1}^m) \big \}  - \log \big \{ \sum_{i=1}^\M h(x,x_i^+) + \M\cdot N \cdot  \mathbb{E}_{x^- \sim p_x^-}  h(x,x^-) \big \}   \geq \varepsilon \bigg ) 
    \\
    &\quad+ \mathbb{P} \bigg ( - \log \big \{ \sum_{i=1}^\M h(x,x_i^+) + \M\cdot N \cdot  G(x, \{u_i\}_{i=1}^n, \{v_i\}_{i=1}^m) \big \} + \log \big \{ \sum_{i=1}^\M h(x,x_i^+) + \M\cdot N \cdot  \mathbb{E}_{x^- \sim p_x^-}  h(x,x^-) \big \}   \geq \varepsilon \bigg ) 
\end{align*}

where the final equality holds simply because $|X| \geq \varepsilon$ if and only if $X \geq \varepsilon$ or $-X \geq \varepsilon$. The first term can be bounded as
\begin{align}
 &\mathbb{P} \bigg (  \log \big \{ \sum_{i=1}^\M h(x,x_i^+) + \M\cdot N \cdot  G(x, \{u_i\}_{i=1}^n, \{v_i\}_{i=1}^m) \big \}  - \log \big \{ \sum_{i=1}^\M h(x,x_i^+) + \M\cdot N \cdot  \mathbb{E}_{x^- \sim p_x^-}  h(x,x^-) \big \}   \geq \varepsilon \bigg ) \nonumber
    \\
 &= \mathbb{P} \bigg (  \log \frac{ \sum_{i=1}^\M h(x,x_i^+) + \M\cdot N \cdot  G(x, \{u_i\}_{i=1}^n, \{v_i\}_{i=1}^m)  }{  \sum_{i=1}^\M h(x,x_i^+) + \M\cdot N \cdot  \mathbb{E}_{x^- \sim p_x^-}  h(x,x^-)  }  \geq \varepsilon \bigg ) \nonumber
    \\
    &\leq \mathbb{P} \bigg (\frac{\M\cdot N \cdot  G(x, \{u_i\}_{i=1}^n, \{v_i\}_{i=1}^m) - \M\cdot N \cdot  \mathbb{E}_{x^- \sim p_x^-} h(x,x^-) }{ \sum_{i=1}^\M h(x,x_i^+) + \M\cdot N \cdot  \mathbb{E}_{x^- \sim p_x^-}  h(x,x^-)} \geq \varepsilon \bigg ) \nonumber
    \\
    & = \mathbb{P} \bigg  (G(x, \{u_i\}_{i=1}^n, \{v_i\}_{i=1}^m) - \mathbb{E}_{x^- \sim p_x^-} h(x,x^-)   \geq \varepsilon \bigg \{ \frac{1}{\M\cdot N } \sum_{i=1}^\M h(x,x_i^+) +  \mathbb{E}_{x^- \sim p_x^-}  h(x,x^-) \bigg \} \bigg ) \nonumber
    \\ 
    &\leq \mathbb{P} \bigg (G(x, \{u_i\}_{i=1}^n, \{v_i\}_{i=1}^m) - \mathbb{E}_{x^- \sim p_x^-} h(x,x^-)  \geq \varepsilon e^{-1}\bigg ).\label{a_eq_1}
\end{align}
The first inequality follows by applying the fact that  $\log x \leq x -1 $ for $x > 0$. The second inequality holds since $ \frac{1}{\M\cdot N \cdot } \sum_{i=1}^\M h(x,x_i^+) +  \mathbb{E}_{x^- \sim p_x^-}  h(x,x^-)  \geq e^{-1}$. Next, we move on to bounding the second term, which proceeds similarly, using the same two bounds.
\begin{align}
     &\mathbb{P} \bigg \{ - \log \big ( \sum_{i=1}^\M h(x,x_i^+) + \M\cdot N \cdot  G(x, \{u_i\}_{i=1}^n, \{v_i\}_{i=1}^m) \big \}  +\log \big \{ \sum_{i=1}^\M h(x,x_i^+) + \M\cdot N \cdot  \mathbb{E}_{x^- \sim p_x^-}  h(x,x^-) \big \}   \geq \varepsilon \bigg ) \nonumber
    \\
    &= \mathbb{P} \bigg ( \log \frac{ \sum_{i=1}^\M h(x,x_i^+) + \M\cdot N \cdot  \mathbb{E}_{x^- \sim p_x^-}  h(x,x^-)}{ \sum_{i=1}^\M h(x,x_i^+) + \M\cdot N \cdot  G(x, \{u_i\}_{i=1}^n, \{v_i\}_{i=1}^m)} \geq \varepsilon \bigg ) \nonumber
    \\
    &\leq \mathbb{P} \bigg (\frac{ \M\cdot N \cdot  \mathbb{E}_{x^- \sim p_x^-} h(x,x^-) - \M\cdot N \cdot  G(x, \{u_i\}_{i=1}^n, \{v_i\}_{i=1}^m) }{\sum_{i=1}^\M h(x,x_i^+)+  \M\cdot N \cdot  G(x, \{u_i\}_{i=1}^N, \{v_i\}_{i=1}^M)} \geq \varepsilon \bigg ) \nonumber
    \\
    & = \mathbb{P} \bigg  (  \mathbb{E}_{x^- \sim p_x^-} h(x,x^-)  - G(x, \{u_i\}_{i=1}^n, \{v_i\}_{i=1}^m) \geq \varepsilon  \bigg \{  \frac{1}{\M\cdot N }\sum_{i=1}^\M h(x,x_i^+)  +  G(x, \{u_i\}_{i=1}^n, \{v_i\}_{i=1}^m) \bigg \}  \bigg ) \nonumber
    \\
    &\leq \mathbb{P} \bigg (  \mathbb{E}_{x^- \sim p_x^-} h(x,x^-) - G(x, \{u_i\}_{i=1}^n, \{v_i\}_{i=1}^m)  \geq \varepsilon e^{-1} \bigg ) . \label{a_eq_2}
\end{align} 
Combining equation \eqref{a_eq_1} and equation \eqref{a_eq_2}, we have
\begin{align*}
\mathbb{P}(\Delta \geq \varepsilon) \leq \mathbb{P} \bigg ( \big |G(x, \{u_i\}_{i=1}^n, \{v_i\}_{i=1}^m) - \mathbb{E}_{x^- \sim p_x^-} h(x,x^-) \big |  \geq \varepsilon e^{-1} \bigg ).
\end{align*}
Lastly, one needs to bound the right hand tail probability. This part of the proof remains exactly the same as in~\cite{chuang2020debiased} and is therefore omitted. 
\end{proof}

\subsection{Bridging the asymptotic objective and supervised loss}
\begin{lemma}\label{thm: supervised}
For any embedding $f$, whenever $N \geq K-1$ we have
\[ L_{\textnormal{Sup}}(f)  \leq  L_{\textnormal{Sup}}^\mu(f) \leq \widetilde L_{\textnormal{Unbiased}}^{\M,N}(f).\]
\end{lemma}
\begin{proof}
We first show that $N=K-1$ gives the smallest loss:
\begin{align}
    \widetilde L_{\textnormal{Unbiased}}^{\M,N}(f) \nonumber &= \mathbb{E}_{\substack{x \sim p \\ x_i^+ \sim p_x^+ }} \left [-\log \frac{\sum_{i=1}^\M e^{f(x)^Tf(x_i^+)}}{\sum_{i=1}^\M e^{f(x)^Tf(x_i^+)} + \M \cdot N\mathbb{E}_{x^- \sim p_x^-} e^{f(x)^T f(x^-)}} \right] \\
    &\geq \mathbb{E}_{\substack{x \sim p \\ x_i^+ \sim p_x^+ }} \left[-\log \frac{\sum_{i=1}^\M e^{f(x)^Tf(x_i^+)}}{\sum_{i=1}^\M e^{f(x)^Tf(x_i^+)} + \M \cdot (K-1)\mathbb{E}_{x^- \sim p_x^-} e^{f(x)^T f(x^-)}} \right] \nonumber\\
    &= L_{\textnormal{Unbiased}}^{\M,K-1}(f) \nonumber
\end{align}
%
To show that $ L_{\textnormal{Unbiased}}^{\M, K-1}(f)$ is an upper bound on the supervised loss $L_{\textnormal{sup}}(f)$, we additionally introduce a task specific class distribution $\rho_{\mathcal{T}}$ which is a uniform distribution over all the possible $K$-way classification tasks with classes in $\mathcal{C}$. That is, we consider all the possible task with $K$ distinct classes $\{ c_1, \dots, c_{K} \} \subseteq \mathcal{C}$.
\begin{align}
     &\quad\; L_{\textnormal{Unbiased}}^{\M, K-1}(f) \nonumber
     \\
    &= \mathbb{E}_{\substack{x \sim p \\ x_i^+ \sim p_x^+ }} \left [-\log \frac{\sum_{i=1}^\M e^{f(x)^Tf(x_i^+)}}{\sum_{i=1}^\M e^{f(x)^Tf(x_i^+)} + \M \cdot (K-1)\mathbb{E}_{x^- \sim p_x^-} e^{f(x)^T f(x^-)}} \right ] \nonumber
    \\
    &= \mathbb{E}_{\mathcal{T} \sim \mathcal{D}}   \mathbb{E}_{\substack{c \sim \rho_{\mathcal{T}}; x \sim p(\cdot|c) \\ x_i^+ \sim p(\cdot|c) }} \left [-\log \frac{\sum_{i=1}^\M e^{f(x)^Tf(x_i^+)}}{\sum_{i=1}^\M e^{f(x)^Tf(x_i^+)} + \M \cdot (K-1)\mathbb{E}_{\mathcal{T} \sim \mathcal{D}} \mathbb{E}_{\rho_{\mathcal{T}}(c^- \sim | c^-  \neq h(x))} \mathbb{E}_{x^- \sim p(\cdot|c^-)} e^{f(x)^T f(x^-)}} \right] \nonumber
    \\
    &\geq \mathbb{E}_{\mathcal{T} \sim \mathcal{D}}   \mathbb{E}_{\substack{c \sim \rho_{\mathcal{T}}; x \sim p(\cdot|c) }} \left[-\log \frac{ \sum_{i=1}^\M e^{f(x)^T\mathbb{E}_{x_i^+ \sim p(\cdot|c)} f(x_i^+)} }{ \sum_{i=1}^\M e^{f(x)^T\mathbb{E}_{x_i^+ \sim p(\cdot|c)} f(x_i^+)} + \M \cdot (K-1)\mathbb{E}_{\mathcal{T} \sim \mathcal{D}} \mathbb{E}_{\rho_{\mathcal{T}}(c^- | c^-  \neq h(x))} \mathbb{E}_{x^- \sim p(\cdot|c^-)} e^{f(x)^T f(x^-)}} \right] \nonumber
    \\
    &\geq \mathbb{E}_{\mathcal{T} \sim \mathcal{D}}   \mathbb{E}_{\substack{c \sim \rho_{\mathcal{T}}; x \sim p(\cdot|c) }}\left [-\log \frac{ \sum_{i=1}^\M e^{f(x)^T\mathbb{E}_{x_i^+ \sim p(\cdot|c)} f(x_i^+)}}{ \sum_{i=1}^\M e^{f(x)^T\mathbb{E}_{x_i^+ \sim p(\cdot|c)} f(x_i^+)} + \M \cdot ( K-1)  \mathbb{E}_{\rho_{\mathcal{T}}(c^-  | c^- \neq h(x))} \mathbb{E}_{x^- \sim p(\cdot|c^-)} e^{f(x)^T f(x^-)}} \right] \nonumber
    \\
    &= \mathbb{E}_{\mathcal{T} \sim \mathcal{D}}   \mathbb{E}_{\substack{c \sim \rho_{\mathcal{T}}; x \sim p(\cdot|c) }}\left [-\log \frac{ \M e^{f(x)^T\mathbb{E}_{x^+ \sim p(\cdot|c)} f(x^+)}}{\M e^{f(x)^T\mathbb{E}_{x^+ \sim p(\cdot|c)} f(x^+)} + \M \cdot (K-1) \mathbb{E}_{\rho_{\mathcal{T}}(c^-  | c^-  \neq h(x))} \mathbb{E}_{x^- \sim p(\cdot|c^-)} e^{f(x)^T f(x^-)}} \right] \nonumber
    \\
    &\geq \mathbb{E}_{\mathcal{T} \sim \mathcal{D}}   \mathbb{E}_{\substack{c \sim \rho_{\mathcal{T}}; x \sim p(\cdot|c) }}\left [-\log \frac{ e^{f(x)^T\mathbb{E}_{x^+ \sim p(\cdot|c)} f(x^+)}}{ e^{f(x)^T\mathbb{E}_{x^+ \sim p(\cdot|c)} f(x^+)} + (K-1)  \mathbb{E}_{\rho_{\mathcal{T}}(c^- | c^-  \neq h(x))}  e^{f(x)^T \mathbb{E}_{x^- \sim p(\cdot|c^-)} f(x^-)}} \right] \nonumber
    \\
    &= \mathbb{E}_{\mathcal{T} \sim \mathcal{D}} \mathbb{E}_{\substack{c \sim \rho_{\mathcal{T}}; x \sim p(\cdot|c) }} \left[-\log \frac{\exp(f(x)^T \mu_{c} )}{\exp(f(x)^T \mu_{c} ) +  \sum_{c^- \in \mathcal{T}, c^- \neq c} \exp(f(x)^T  \mu_{c^-} )} \right ]  \nonumber\label{eq_lemma_3_1}
    \\
    &= \mathbb{E}_{\mathcal{T} \sim \mathcal{D}} L_{\textnormal{Sup}}^\mu(\mathcal{T}, f) \nonumber
    \\
    &= \bar L_{\textnormal{Sup}}^\mu(f) \nonumber
\end{align}
where the three inequalities follow from Jensen's inequality. The first and third inequality shift the expectations $\mathbb{E}_{x^+ \sim p_{x, \mathcal{T}}^+}$ and $\mathbb{E}_{x^- \sim p(\cdot|c^-)}$, respectively, via the convexity of the functions and the second moves the expectation $\mathbb{E}_{\mathcal{T} \sim \mathcal{D}}$ out using concavity. Note that $\bar L_{\textnormal{Sup}}(f)  \leq \bar L_{\textnormal{Sup}}^\mu(f)$ holds trivially.
\end{proof}

\subsection{Generalization bounds}

We wish to derive a data dependent bound on the downstream supervised generalization error of the debiased contrastive objective. Recall that a sample $(x,\{x_i^+\}_{i=1}^\M, \{u_i\}_{i=1}^n,  \{v_i\}_{i=1}^m )$ yields loss 
\begin{align*}
 - \log  \left \{ \frac{\sum_{i=1}^\M e^{f(x)^\top f(x_i^+)} } { \sum_{i=1}^\M e^{f(x)^\top f(x_i^+)} + \M \cdot N \cdot G(x, \{u_i\}_{i=1}^n,  \{v_i\}_{i=1}^m)} \right \} &=   \log \left \{  1 + \M \cdot N \frac{ G(x, \{u_i\}_{i=1}^n,  \{v_i\}_{i=1}^m)}{\sum_{i=1}^\M e^{f(x)^\top f(x_i^+)}} \right \} ,
\end{align*}
which is equal to $  \ell \left (  \left \{ \frac{e^{f(x)^\top f(u_j)}}{\sum_{i=1}^\M e^{f(x)^\top f(x_i^+)} } \right \}_{j=1}^n , \left \{ \frac{e^{f(x)^\top f(v_j)}}{\sum_{i=1}^\M e^{f(x)^\top f(x_i^+)} } \right \}_{j=1}^m \right ) $, where we define
\begin{align*}
    \ell( \{ a_i \}_{i=1}^n ,  \{ b_i \}_{i=1}^m ) &\coloneqq  \log \left \{  1 + \M \cdot N \max \left ( \frac{1}{\tau^-} \frac{1}{n} \sum_{i=1}^n  a_i - \frac{\tau^+}{\tau^-} \frac{1}{m} \sum_{i=1}^m b_i   , e^{-1} \right ) \right \} \\
    \widehat{L}_{\text{Debiased}}^{M,N,n,m}(f) &\coloneqq \frac{1}{T}\sum_{t=1}^T  \ell \left (  \left \{ \frac{e^{f(x_t)^\top f(u_{tj})}}{\sum_{i=1}^\M e^{f(x_t)^\top f(x_{ti}^+)} } \right \}_{j=1}^n , \left \{ \frac{e^{f(x_t)^\top f(v_{tj})}}{\sum_{i=1}^\M e^{f(x_t)^\top f(x_{ti}^+)} } \right \}_{j=1}^m \right ) \\
    \hat{f} &\coloneqq \argmin_{f \in \mathcal{F}} \widehat{L}_{\text{Debiased}}^{M,N,n,m}(f) 
\end{align*}

\begin{theorem}
With probability at least $1-\delta$, for all $f \in \mathcal{F}$ and $N \geq K-1$,
 \begin{align}
      L_{\textnormal{Sup}}(\hat{f}) \leq L_{\substack{\textnormal{Debiased}}}^{\M,N,n,m}(f) + \mathcal{O} \left (\frac{1}{\tau^-}\sqrt{\frac{1}{n}} + \frac{\tau^+}{\tau^-}\sqrt{\frac{1}{m}}+  \frac{\lambda \mathcal{R}_\mathcal{S}(\mathcal{F})}{T}+  B\sqrt{\frac{\log{\frac{1}{\delta}}}{T}} \right ), \nonumber 
 \end{align}
 where $\lambda = \frac{1}{\M}\sqrt{\frac{1}{{\tau^-}^2} (\frac{m}{n}+1) + {\tau^+}^2 (\frac{n}{m}+1)}$ and $B = \log N \left ( \frac{1}{\tau^-} + \tau^+ \right)$.
\end{theorem}

\begin{proof}
Considering the samples to be $\left\{\left(x_t, \left\{ x_{ti}^+ \right\}_{i=1}^\M, \left\{u_{ti}\right\}_{i=1}^n, \left\{v_{ti}\right\}_{i=1}^m\right)\right\}_{t=1}^T$. Then, we can use the standard bounds for empirical versus population means of any $B-$bounded function $g$ belonging to a function class $G$, we have that with probability at least $1-\frac{\delta}{2}$.
\begin{equation} \label{eq: generalize}
    \mathbb{E}[g(x)] \leq \frac{1}{T}\sum_{t=1}^T g(x_i) + \frac{2\mathcal{R}_S(G)}{T} + 3B\sqrt{\frac{\log\left(\frac{4}{\delta}\right)}{2T}}
\end{equation}
 
In order to calculate $\mathcal{R}_S(G)$ we use the same trick as in \cite{saunshi2019theoretical}. We express it as a composition of functions $g = \ell\left(h\left(f\left(x_t, \left\{ x_{ti}^+ \right\}_{i=1}^\M, \left\{u_{ti}\right\}_{i=1}^n, \left\{v_{ti}\right\}_{i=1}^m\right)\right)\right)$ where $f \in \mathcal{F}$ just maps each sample to corresponding feature vector and $h$ maps the feature vectors to the $\{a\}_{i=1}^n, \{b\}_{i=1}^m$. Then we use contraction inequality to bound $\mathcal{R}_S(G)$ with $\mathcal{R}_S(\mathcal{F})$. In order to do this we need to compute the Lipschitz constant for the intermediate function $h$ in the composition.

For $h$, we see that the Jacobian has the following form 

\[ \frac{\partial a_i}{\partial f(x)} = a_i \frac{\sum_{j=1}^\M (f(u_i) - f(x_j))e^{f(x)^\top f(x_j^+)}}{\sum_{j=1}^\M e^{f(x)^\top f(x_j^+)}}; \quad \frac{\partial b_i}{\partial f(x)} = b_i \frac{\sum_{j=1}^\M (f(v_i) - f(x_j)) e^{f(x)^\top f(x_j^+)}}{\sum_{j=1}^\M e^{f(x)^\top f(x_j^+)}} \]
\[ \frac{\partial a_i}{\partial f(x_j^+)} = - a_i \frac{f(x) e^{f(x)^\top f(x_j^+)}}{\sum_{k=1}^\M e^{f(x)^\top f(x_k^+)}}; \quad \frac{\partial b_i}{\partial f(x_j^+)} = - b_i \frac{f(x) e^{f(x)^\top f(x_j^+)}}{\sum_{k=1}^\M e^{f(x)^\top f(x_k^+)}} \]
\[ \frac{\partial a_i}{\partial f(u_j)} = f(x) a_i \delta(i - j); \quad  \frac{\partial b_i}{\partial f(v_j)} = f(x) b_i \delta(i - j) \]
 
Using the fact that $\norm{f(\cdot)}_2 = 1$, we get $\frac{e^{-2}}{\M} \leq a_i, b_i \leq \frac{e^2}{\M}$ and
 \begin{align*}
     \norm{J}^2_2 \leq \norm{J}^2_F &\leq \sum_{i=1}^n a_i^2 \left( \norm{\frac{\sum_{j=1}^\M (f(u_i) - f(x_j))e^{f(x)^\top f(x_j^+)}}{\sum_{j=1}^\M e^{f(x)^\top f(x_j^+)}}}_2^2 +
     \norm{f(x)}_2^2 \frac{\sum_{j=1}^\M e^{2f(x)^\top f(x_j^+)}}{\left(\sum_{j=1}^\M e^{f(x)^\top f(x_j^+)}\right)^2} + \norm{f(x)}_2^2 \right) \\
     &+ \sum_{i=1}^m b_i^2 \left( \norm{\frac{\sum_{j=1}^\M (f(v_i) - f(x_j))e^{f(x)^\top f(x_j^+)}}{\sum_{j=1}^\M e^{f(x)^\top f(x_j^+)}} }_2^2 + \norm{f(x)}_2^2 \frac{\sum_{j=1}^\M e^{2f(x)^\top f(x_j^+)}}{\left(\sum_{j=1}^\M e^{f(x)^\top f(x_j^+)}\right)^2} + \norm{f(x)}_2^2 \right) \\
     &\leq \sum_{i=1}^n a_i^2\left(4 + 1 + 1\right) + \sum_{i=1}^m b_i^2\left(4 + 1 + 1\right) \leq \frac{6(n+m)e^4}{\M^2}
 \end{align*}
 
 Using this and the Lipschitz constant, $O\left(\sqrt{\frac{1}{n{\tau^-}^2} + \frac{{\tau^+}^2}{m}}\right)$ of $\ell$ derived in \cite{chuang2020debiased}, we get $\mathcal{R}_S(\mathcal{G}) = \lambda \mathcal{R}_S(\mathcal{F})$ where $\lambda = \mathcal{O}\left(\frac{1}{\M}\sqrt{\frac{1}{{\tau^-}^2} (\frac{m}{n}+1) + {\tau^+}^2 (\frac{n}{m}+1)}\right)$. From \cite{chuang2020debiased}, we also get $B = O\left(\log N \left ( \frac{1}{\tau^-} + \tau^+ \right)\right)$. Combining this with Equation \ref{eq: generalize} gives us that with probability at least $1- \frac{\delta}{2}$
 \[ L_{\substack{\textnormal{Debiased}}}^{\M,N,n,m}(\hat{f}) \leq \widehat{L}_{\substack{\textnormal{Debiased}}}^{\M,N,n,m}(\hat{f}) + \mathcal{O} \left ( \frac{\lambda \mathcal{R}_\mathcal{S}(\mathcal{F})}{T}+  B\sqrt{\frac{\log{\frac{1}{\delta}}}{T}} \right )\]
 
 Using Theorem \ref{thm: bound Delta}, we get that 
 \begin{align*}
  L_{\substack{\textnormal{Unbiased}}}^{\M,N}(\hat{f}) &\leq  L_{\substack{\textnormal{Debiased}}}^{\M,N,n,m}(\hat{f}) + \mathcal{O} \left (\frac{1}{\tau^-}\sqrt{\frac{1}{n}} + \frac{\tau^+}{\tau^-}\sqrt{\frac{1}{m}} \right) \\
  &\leq \widehat{L}_{\substack{\textnormal{Debiased}}}^{\M,N,n,m}(\hat{f}) + \mathcal{O} \left (\frac{1}{\tau^-}\sqrt{\frac{1}{n}} + \frac{\tau^+}{\tau^-}\sqrt{\frac{1}{m}} + \frac{\lambda \mathcal{R}_\mathcal{S}(\mathcal{F})}{T}+  B\sqrt{\frac{\log{\frac{1}{\delta}}}{T}} \right )
 \end{align*}
 Using Lemma \ref{thm: supervised}, we get
  \begin{align*}
  L_{\textnormal{Sup}}(\hat{f}) &\leq L_{\substack{\textnormal{Unbiased}}}^{\M,N}(\hat{f}) \leq \widehat{L}_{\substack{\textnormal{Debiased}}}^{\M,N,n,m}(\hat{f}) + \mathcal{O} \left (\frac{1}{\tau^-}\sqrt{\frac{1}{n}} + \frac{\tau^+}{\tau^-}\sqrt{\frac{1}{m}} + \frac{\lambda \mathcal{R}_\mathcal{S}(\mathcal{F})}{T}+  B\sqrt{\frac{\log{\frac{1}{\delta}}}{T}} \right )
 \end{align*}
 Finally we see that for any $f$, we can use M Hoeffding's inequality to show that with at least $1 - \frac{\delta}{2}$ probability
 $$\widehat{L}_{\substack{\textnormal{Debiased}}}^{\M,N,n,m}(f) \leq L_{\substack{\textnormal{Debiased}}}^{\M,N,n,m}(f) + 3B\sqrt{\frac{\log(2\over\delta)}{2T}}$$
 Combining all of the above results gives us that with probability at least $1-\delta$,
 \begin{align*}
     L_{\textnormal{Sup}}(\hat{f}) &\leq L_{\substack{\textnormal{Unbiased}}}^{\M,N}(\hat{f}) \leq \widehat{L}_{\substack{\textnormal{Debiased}}}^{\M,N,n,m}(\hat{f}) + \mathcal{O} \left (\frac{1}{\tau^-}\sqrt{\frac{1}{n}} + \frac{\tau^+}{\tau^-}\sqrt{\frac{1}{m}} + \frac{\lambda \mathcal{R}_\mathcal{S}(\mathcal{F})}{T}+  B\sqrt{\frac{\log{\frac{1}{\delta}}}{T}} \right )\\
  &\leq \widehat{L}_{\substack{\textnormal{Debiased}}}^{\M,N,n,m}(f) + \mathcal{O} \left (\frac{1}{\tau^-}\sqrt{\frac{1}{n}} + \frac{\tau^+}{\tau^-}\sqrt{\frac{1}{m}} + \frac{\lambda \mathcal{R}_\mathcal{S}(\mathcal{F})}{T}+  B\sqrt{\frac{\log{\frac{1}{\delta}}}{T}} \right ) \\
  &\leq L_{\substack{\textnormal{Debiased}}}^{\M,N,n,m}(f) + \mathcal{O} \left (\frac{1}{\tau^-}\sqrt{\frac{1}{n}} + \frac{\tau^+}{\tau^-}\sqrt{\frac{1}{m}} + \frac{\lambda \mathcal{R}_\mathcal{S}(\mathcal{F})}{T}+  B\sqrt{\frac{\log{\frac{1}{\delta}}}{T}} \right ) + \mathcal{O}\left(B \sqrt{\frac{\log (1\over\delta)}{T}}\right)
 \end{align*}
\end{proof}

\newpage
\section{Table of Definitions}
\begin{table}[h]
    \centering
    \caption{A summary of definitions. } 
\label{tbl:summary_2}
\begin{tabular}{l|l}
\hlineB{3}
$\mathcal{L}_{\text{NCA}}(G^1,\M)$ & $\E_{x\sim \D, x_j^+\sim \Dp, x_i^-\sim \Dm}
[-\log\frac{\sum\limits_{j=1}\limits^\M e^{f(x)^Tf(x_j^+)}}{\sum\limits_{j=1}\limits^\M e^{f(x)^Tf(x_j^+)}+N  G^1(x,\{x_{i}^-\}^N)}]$\\
 & $\E_{x\sim \D, x^+\sim \Dp, x_{i_1}^-, x_{i_2j}^-, x_{j}^-\sim \Dm}
[-\log\frac{e^{f(x)^Tf(x^+)}}{e^{f(x)^Tf(x^+)}+N  G^1(x,\{x_{i_1}^-\}^N)}$\\
\multirow{3}{*}{$\mathcal{L}_{\text{MIXNCA}}(G^1,\M,\lambda)$} & $- \frac{\lambda}{\M-1}\sum\limits_{j=1}\limits^{\M-1}
\log\frac{e^{f(x)^Tf(\lambda x^+ +(1-\lambda)x^-_{j})}}{e^{f(x)^Tf(\lambda x^++(1-\lambda)x^-_{j})}+N  G^1(x,\{x_{i_2 j}^-\}^N_{i_2})}$\\
& $- \frac{1-\lambda}{\M-1}\sum\limits_{j=1}\limits^{\M-1}
\log(1-\frac{e^{f(x)^Tf(\lambda x^+ +(1-\lambda)x^-_{j})}}{e^{f(x)^Tf(\lambda x^++(1-\lambda)x^-_{j})}+N  G^1(x,\{x_{i_2 j}^-\}^N_{i_2})})]$\\\hline
$g_0(x,\{x_i^-\}^N_i)$ & $\frac{1}{N}\sum_{i=1}^N e^{f(x)^Tf(x^-_i)}$\\
$g_1(x,\{u_i\}^n,\{v_j\}^m)$ & $\max\{\frac{1}{1-\tau^+}(\frac{1}{n}\sum_{i=1}^n e^{f(x)^Tf(u_i)}-\tau^+ \frac{1}{m}\sum_{j=1}^m e^{f(x)^Tf(v_j)}),e^{-1/t}\}$\\
$g_2(x,\{u_i\}^n,\{v_j\}^m)$ & $\max\{\frac{1}{1-\tau^+}(\frac{\sum_{i=1}^n e^{(\beta+1)f(x)^Tf(u_i)}}{\sum_{i=1}^n e^{\beta f(x)^Tf(u_i)}}-\tau^+ \frac{1}{m}\sum_{j=1}^m e^{f(x)^Tf(v_j)}),e^{-1/t}\}$\\\hline
$\hat{w}(x)$ & $-\log\frac{e^{f(x)^Tf(x^+)}}{e^{f(x)^Tf(x^+)}+N  G(x,\cdot)}$\\\hlineB{3}
\end{tabular}
\end{table}

\newpage
\section{Complete Tables of Results}
\label{sec:complete}
We give the full table of results in Section~\ref{sec:exp} in the following. Notably, we gather the standard accuracy, robust accuracy, transfer accuracy, and transfer robust accuracy for each specification.
\begin{table*}[h!]
    \centering
    \caption{The effectiveness evaluation of NaCl on SimCLR (i.e. $\alpha=0, G^1=g_0$). The best performance within each loss type is in boldface. }
    \label{tab:NCA_SimCLR}
    \scalebox{1}
    {\begin{tabular}{c|cccc}
      \toprule
    \multirow{2}{*}{$\M$} & \multicolumn{4}{c}{$\alpha=0,~\mathcal{L}_{\text{NaCl}}(G^1, \M, \lambda)=\mathcal{L}_{\text{NCA}}(g_0, \M)$}\\
        &  CIFAR100 Acc. & FGSM Acc. & CIFAR10 Acc. & FGSM Acc.\\\midrule
      1  & 53.69$\pm$0.25 & 25.17$\pm$0.55 & 76.34$\pm$0.28 & 43.50$\pm$0.41\\
      2  & 55.72$\pm$0.15 & 27.04$\pm$0.45 & 77.40$\pm$0.14 & 44.58$\pm$0.41\\
      3  & 56.67$\pm$0.12 & \textbf{28.41$\pm$0.24} & 77.53$\pm$0.24 & \textbf{45.21$\pm$0.89}\\
      4  & 57.09$\pm$0.26 & 28.20$\pm$0.81 & 77.75$\pm$0.22 & 45.13$\pm$0.44\\
      5  & \textbf{57.32$\pm$0.17} & 28.33$\pm$0.59 & \textbf{77.93$\pm$0.40} & 44.46$\pm$0.53\\\midrule
    & \multicolumn{4}{c}{$\alpha=0,~\mathcal{L}_{\text{NaCl}}(G^1, \M, \lambda)=\mathcal{L}_{\text{MIXNCA}}(g_0, \M, 0.5)$}\\\midrule
      1 & 53.69$\pm$0.25 & \textbf{25.17$\pm$0.55} & 76.34$\pm$0.28 & \textbf{43.50$\pm$0.41} \\
      2 & 54.76$\pm$0.29 & 23.66$\pm$0.27 & 76.78$\pm$0.26 & 40.76$\pm$0.66\\
      3 & 55.21$\pm$0.17 & 24.46$\pm$0.44 & 77.45$\pm$0.18 & 41.78$\pm$0.80\\
      4 & 55.68$\pm$0.27 & 24.19$\pm$0.46 & 77.40$\pm$0.24 & 41.33$\pm$0.34\\
      5 & \textbf{55.85$\pm$0.16} & 24.01$\pm$0.91 & \textbf{77.50$\pm$0.16} & 40.77$\pm$0.66\\\midrule
       & \multicolumn{4}{c}{$\alpha=0,~\mathcal{L}_{\text{NaCl}}(G^1, \M, \lambda)=\mathcal{L}_{\text{MIXNCA}}(g_0, \M, 0.6)$}\\\midrule 
      1  & 53.69$\pm$0.25 & 25.17$\pm$0.55 & 76.34$\pm$0.28 & \textbf{43.50$\pm$0.41}\\
      2 & 54.84$\pm$0.35 & 25.94$\pm$0.81 & 77.11$\pm$0.15 & 42.81$\pm$0.83\\
      3 & 55.49$\pm$0.13 & \textbf{26.25$\pm$0.89} & 76.95$\pm$0.32 & 42.99$\pm$0.96\\
      4 & 55.65$\pm$0.24 & 25.41$\pm$0.53 & \textbf{77.39$\pm$0.37} & 42.69$\pm$1.20\\
      5 & \textbf{55.66$\pm$0.22} & 26.01$\pm$0.60 & 77.26$\pm$0.48 & 43.06$\pm$0.79\\\midrule
      & \multicolumn{4}{c}{$\alpha=0,~ \mathcal{L}_{\text{NaCl}}(G^1, \M, \lambda)=\mathcal{L}_{\text{MIXNCA}}(g_0, \M, 0.7)$} \\\midrule
      1 & 53.69$\pm$0.25 & 25.17$\pm$0.55 & 76.34$\pm$0.28 & 43.50$\pm$0.41 \\
      2 & 55.57$\pm$0.32 & 27.67$\pm$0.60 & 77.09$\pm$0.27 & 44.68$\pm$0.71\\
      3 & 55.83$\pm$0.25 & 27.72$\pm$0.59 & 77.23$\pm$0.28 & 43.68$\pm$0.72\\
      4 & 56.29$\pm$0.25 & \textbf{27.92$\pm$0.60} & 77.33$\pm$0.29 & 44.69$\pm$0.82\\
      5 & \textbf{56.37$\pm$0.32} & 27.78$\pm$0.54 & \textbf{77.40$\pm$0.20} & \textbf{45.07$\pm$0.98}\\\midrule 
      & \multicolumn{4}{c}{$\alpha=0,~\mathcal{L}_{\text{NaCl}}(G^1, \M, \lambda)=\mathcal{L}_{\text{MIXNCA}}(g_0, \M, 0.8)$}\\\midrule
      1 & 53.69$\pm$0.25 & 25.17$\pm$0.55 & 76.34$\pm$0.28 & 43.50$\pm$0.41 \\
      2 & 55.75$\pm$0.21 & 29.30$\pm$0.86 & 76.80$\pm$0.20 & 46.56$\pm$1.02\\
      3 & 56.27$\pm$0.26 & \textbf{29.96$\pm$0.29} & 77.11$\pm$0.37 & 46.52$\pm$0.50\\
      4 & \textbf{56.39$\pm$0.26} & 29.49$\pm$0.65 & 77.34$\pm$0.31 & 46.79$\pm$0.93\\
      5 & 56.23$\pm$0.13 & 29.47$\pm$0.95 & \textbf{77.40$\pm$0.14} & \textbf{47.36$\pm$0.69}\\\midrule
      & \multicolumn{4}{c}{$\alpha=0,~\mathcal{L}_{\text{NaCl}}(G^1, \M, \lambda)=\mathcal{L}_{\text{MIXNCA}}(g_0, \M, 0.9)$} \\\midrule
      1 & 53.69$\pm$0.25 & 25.17$\pm$0.55 & 76.34$\pm$0.28 & 43.50$\pm$0.41  \\
      2 & 56.20$\pm$0.33 & 30.95$\pm$0.36 & 76.96$\pm$0.15 & \textbf{48.85$\pm$0.75}\\
      3 & 56.41$\pm$0.13 & \textbf{30.98$\pm$0.90} & 77.10$\pm$0.21 & 48.76$\pm$0.63\\
      4 & 56.00$\pm$0.42 & 29.90$\pm$0.63 & \textbf{77.11$\pm$0.40} & 48.16$\pm$0.40\\
      5 & \textbf{56.63$\pm$0.31} & 30.58$\pm$0.52 & 77.04$\pm$0.19 & 47.96$\pm$0.46\\\bottomrule
    \end{tabular}
    }
\end{table*}
\newpage
\begin{table*}[h!]
    \centering
    \caption{The effectiveness evaluation of NaCl on Debised+HardNeg (i.e. $\alpha=0, G^1=g_2$). The best performance within each loss type is in boldface.}
    \label{tab:NCA_DN}
    \scalebox{1}
    {\begin{tabular}{c|cccc}
    \toprule
    \multirow{2}{*}{$\M$} & \multicolumn{4}{c}{$\alpha=0,~\mathcal{L}_{\text{NaCl}}(G^1, \M, \lambda)=\mathcal{L}_{\text{NCA}}(g_2, \M)$} \\
      &  CIFAR100 Acc. & FGSM Acc. & CIFAR10 Acc. & FGSM Acc. \\\midrule
      1 & 56.83$\pm$0.20 & 31.03$\pm$0.41 & 77.24$\pm$0.29 & 48.38$\pm$0.70\\
      2 & 57.87$\pm$0.15 & 32.50$\pm$0.48 & 77.43$\pm$0.11 & 48.14$\pm$0.31\\
      3 & 58.42$\pm$0.23 & \textbf{33.19$\pm$0.60} & 77.41$\pm$0.17 & 48.09$\pm$0.93\\
      4 & \textbf{58.86$\pm$0.18} & 32.65$\pm$1.07 & 77.46$\pm$0.29 & \textbf{48.43$\pm$0.94}\\
      5 & 58.81$\pm$0.21 & 32.86$\pm$0.47 & \textbf{77.58$\pm$0.23} & 48.30$\pm$0.39\\\midrule
    & \multicolumn{4}{c}{$\alpha=0,~\mathcal{L}_{\text{NaCl}}(G^1, \M, \lambda)=\mathcal{L}_{\text{MIXNCA}}(g_2, \M, 0.5)$} \\\midrule
      1 & 56.83$\pm$0.20 & 31.03$\pm$0.41 & 77.24$\pm$0.29 & 48.38$\pm$0.70 \\
      2 & 59.41$\pm$0.19 & \textbf{32.22$\pm$0.35} & 79.36$\pm$0.65 & 48.86$\pm$0.34 \\
      3 & 59.81$\pm$0.25 & 32.04$\pm$0.67 & 79.41$\pm$0.17 & 48.91$\pm$0.81 \\
      4 & 59.75$\pm$0.33 & 32.03$\pm$0.34 & 79.42$\pm$0.18 & \textbf{49.05$\pm$0.71} \\
      5 & \textbf{59.85$\pm$0.30} & 32.06$\pm$0.72 & \textbf{79.45$\pm$0.20} & 48.32$\pm$0.70 \\
      \midrule 
      & \multicolumn{4}{c}{$\alpha=0,~\mathcal{L}_{\text{NaCl}}(G^1, \M, \lambda)=\mathcal{L}_{\text{MIXNCA}}(g_2, \M, 0.6)$} \\\midrule
      1 & 56.83$\pm$0.20 & 31.03$\pm$0.41 & 77.24$\pm$0.29 & 48.38$\pm$0.70\\
      2 & 58.94$\pm$0.29 & 32.65$\pm$0.36 & 78.67$\pm$0.15 & \textbf{49.86$\pm$0.59}\\
      3 & 59.43$\pm$0.35 & 32.91$\pm$0.40 & 78.94$\pm$0.19 & 48.84$\pm$1.09\\
      4 & \textbf{59.54$\pm$0.28} & 33.02$\pm$0.62 & 78.92$\pm$0.29 & 49.64$\pm$0.74\\
      5 & 59.52$\pm$0.28 & \textbf{33.10$\pm$0.50} & \textbf{79.29$\pm$0.21} & 49.39$\pm$1.02\\\midrule
      & \multicolumn{4}{c}{$\alpha=0,~\mathcal{L}_{\text{NaCl}}(G^1, \M, \lambda)=\mathcal{L}_{\text{MIXNCA}}(g_2, \M, 0.7)$} \\\midrule
      1 & 56.83$\pm$0.20 & 31.03$\pm$0.41 & 77.24$\pm$0.29 & 48.38$\pm$0.70 \\
      2 & 58.24$\pm$0.19 & 33.24$\pm$0.90 & 78.30$\pm$0.31 & \textbf{50.40$\pm$0.83}\\
      3 & 58.74$\pm$0.26 & 33.12$\pm$0.59 & 78.49$\pm$0.30 & 49.85$\pm$0.38\\
      4 & 58.79$\pm$0.38 & \textbf{33.63$\pm$0.53} & 78.51$\pm$0.29 & 49.88$\pm$0.75\\
      5 & \textbf{58.99$\pm$0.18} & 32.93$\pm$0.81 & \textbf{78.57$\pm$0.12} & 49.53$\pm$1.55\\\midrule 
      & \multicolumn{4}{c}{$\alpha=0,~\mathcal{L}_{\text{NaCl}}(G^1, \M, \lambda)=\mathcal{L}_{\text{MIXNCA}}(g_2, \M, 0.8)$} \\\midrule
      1 & 56.83$\pm$0.20 & 31.03$\pm$0.41 & 77.24$\pm$0.29 & 48.38$\pm$0.70 \\
      2 & 57.60$\pm$0.15 & 34.14$\pm$0.22 & \textbf{77.96$\pm$0.07} & \textbf{51.82$\pm$0.68}\\
      3 & 58.04$\pm$0.28 & 33.93$\pm$0.45 & 77.55$\pm$0.18 & 50.30$\pm$0.81\\
      4 & 58.05$\pm$0.16 & \textbf{34.16$\pm$0.54} & 77.90$\pm$0.21 & 50.40$\pm$0.43\\
      5 & \textbf{58.43$\pm$0.27} & 33.87$\pm$0.62 & 77.90$\pm$0.17 & 50.78$\pm$0.95\\\midrule
      & \multicolumn{4}{c}{$\alpha=0,~\mathcal{L}_{\text{NaCl}}(G^1, \M, \lambda)=\mathcal{L}_{\text{MIXNCA}}(g_2, \M, 0.9)$} \\\midrule
      1 & 56.83$\pm$0.20 & 31.03$\pm$0.41 & 77.24$\pm$0.29 & 48.38$\pm$0.70 \\
      2 & 57.16$\pm$0.15 & 34.25$\pm$0.55 & 77.19$\pm$0.09 & \textbf{51.42$\pm$0.45}\\
      3 & 57.08$\pm$0.10 & 33.96$\pm$0.19 & 77.21$\pm$0.26 & 51.30$\pm$1.05 \\
      4 & 57.36$\pm$0.19 & \textbf{34.29$\pm$0.15} & \textbf{77.34$\pm$0.34} & 51.16$\pm$0.55 \\
      5 & \textbf{57.38$\pm$0.16} & 34.25$\pm$0.30 & 77.13$\pm$0.16 & 50.68$\pm$0.74 \\\bottomrule
    \end{tabular}}
\end{table*}
\newpage
\begin{table*}[h!]
    \centering
    \caption{The effectiveness evaluation of NaCl ($\M\neq 1$) on IntCl ($\M= 1$) when $\alpha=1, G^1=G^2=g_2$. The best performance within each loss type is in boldface.}
    \label{tab:sota}
    \scalebox{1}
    {\begin{tabular}{c|cccc}
    \toprule
    \multirow{2}{*}{$M$} & \multicolumn{4}{c}{$\alpha\neq0,~\mathcal{L}_{\text{NaCl}}(G^1, \M, \lambda)=\mathcal{L}_{\text{NCA}}(g_2, \M)$} \\
      &  CIFAR100 Acc. & FGSM Acc. & CIFAR10 Acc. & FGSM Acc. \\\midrule
      1 & 56.22$\pm$0.15 & 40.05$\pm$0.67 & 76.39$\pm$0.10 & \textbf{59.33$\pm$0.94}\\
      2 & 56.71$\pm$0.11 & 39.80$\pm$0.57 & 76.55$\pm$0.27 & 58.44$\pm$0.31\\
      3 & 57.13$\pm$0.26 & 40.53$\pm$0.29 & \textbf{76.67$\pm$0.22} & 58.47$\pm$0.31\\
      4 & 57.06$\pm$0.19 & 40.85$\pm$0.31 & 76.34$\pm$0.22 & 58.91$\pm$0.62\\
      5 & \textbf{57.46$\pm$0.04} & \textbf{41.00$\pm$0.86} & 76.60$\pm$0.37 & 57.98$\pm$0.47\\\midrule
      & \multicolumn{4}{c}{$\alpha\neq0,~\mathcal{L}_{\text{NaCl}}(G^1, \M, \lambda)=\mathcal{L}_{\text{MIXNCA}}(g_2, \M,0.5)$} \\\midrule
      1 & 56.22$\pm$0.15 & 40.05$\pm$0.67 & 76.39$\pm$0.10 & 59.33$\pm$0.94 \\
      2 & 58.97$\pm$0.19 & 40.25$\pm$0.52 & 78.61$\pm$0.20 & 58.41$\pm$0.59 \\
      3 & 59.26$\pm$0.18 & 40.96$\pm$0.58 & \textbf{78.83$\pm$0.22} & 59.20$\pm$1.25 \\
      4 & 59.32$\pm$0.21 & 40.82$\pm$0.54 & 78.83$\pm$0.27 & 59.03$\pm$0.52 \\
      5 & \textbf{59.43$\pm$0.23} & \textbf{41.01$\pm$0.34} & 78.80$\pm$0.21 & \textbf{59.51$\pm$0.93} \\\midrule 
      & \multicolumn{4}{c}{$\alpha\neq0,~\mathcal{L}_{\text{NaCl}}(G^1, \M, \lambda)=\mathcal{L}_{\text{MIXNCA}}(g_2, \M,0.6)$} \\\midrule
      1 & 56.22$\pm$0.15 & 40.05$\pm$0.67 & 76.39$\pm$0.10 & 59.33$\pm$0.94\\
      2 & 58.55$\pm$0.34 & \textbf{40.85$\pm$0.62} & 78.34$\pm$0.22 & \textbf{59.56$\pm$0.88}\\
      3 & 59.05$\pm$0.21 & 40.83$\pm$0.44 & 78.41$\pm$0.12 & 59.14$\pm$0.78\\
      4 & 59.06$\pm$0.25 & 40.80$\pm$0.89 & 78.61$\pm$0.22 & 58.41$\pm$1.00\\
      5 & \textbf{59.10$\pm$0.23} & 40.68$\pm$0.50 & \textbf{78.63$\pm$0.21} & 58.92$\pm$0.76 \\\midrule 
      & \multicolumn{4}{c}{$\alpha\neq0,~\mathcal{L}_{\text{NaCl}}(G^1, \M, \lambda)=\mathcal{L}_{\text{MIXNCA}}(g_2, \M,0.7)$} \\\midrule
      1 & 56.22$\pm$0.15 & 40.05$\pm$0.67 & 76.39$\pm$0.10 & 59.33$\pm$0.94 \\
      2 & 58.00$\pm$0.18 & 40.35$\pm$0.34 & 77.73$\pm$0.24 & 59.40$\pm$1.27\\
      3 & 58.23$\pm$0.18 & 40.94$\pm$0.75 & 77.91$\pm$0.25 & \textbf{59.57$\pm$0.81}\\
      4 & 58.20$\pm$0.25 & 40.95$\pm$0.45 & 77.89$\pm$0.20 & 59.49$\pm$0.49\\
      5 & \textbf{58.37$\pm$0.14} & \textbf{41.15$\pm$0.48} & \textbf{78.27$\pm$0.26} & 59.17$\pm$0.94\\\midrule 
      & \multicolumn{4}{c}{$\alpha\neq0,~\mathcal{L}_{\text{NaCl}}(G^1, \M, \lambda)=\mathcal{L}_{\text{MIXNCA}}(g_2, \M,0.8)$} \\\midrule
      1 & 56.22$\pm$0.15 & 40.05$\pm$0.67 & 76.39$\pm$0.10 & 59.33$\pm$0.94\\
      2 & 57.07$\pm$0.24 & 41.29$\pm$0.57 & 77.27$\pm$0.28 & 60.16$\pm$0.51\\
      3 & \textbf{57.62$\pm$0.22} & 40.93$\pm$0.49 & 77.54$\pm$0.27 & 59.47$\pm$0.52\\
      4 & 57.61$\pm$0.25 & \textbf{41.36$\pm$0.41} & 77.50$\pm$0.34 & \textbf{60.28$\pm$0.68}\\
      5 & 57.56$\pm$0.18 & 40.71$\pm$0.34 & \textbf{77.58$\pm$0.42} & 59.99$\pm$0.30\\\midrule
      & \multicolumn{4}{c}{$\alpha\neq0,~\mathcal{L}_{\text{NaCl}}(G^1, \M, \lambda)=\mathcal{L}_{\text{MIXNCA}}(g_2, \M,0.9)$} \\\midrule
      1 & 56.22$\pm$0.15 & 40.05$\pm$0.67 & 76.39$\pm$0.10 & 59.33$\pm$0.94  \\
      2 & 56.54$\pm$0.33 & 40.85$\pm$0.13 & 76.81$\pm$0.22 & 60.40$\pm$0.46\\
      3 & 56.69$\pm$0.11 & 41.23$\pm$0.66 & \textbf{76.98$\pm$0.22} & 60.13$\pm$0.56\\
      4 & 56.43$\pm$0.26 & \textbf{41.56$\pm$0.56} & 76.97$\pm$0.20 & \textbf{61.21$\pm$0.49}\\
      5 & \textbf{56.86$\pm$0.11} & 41.09$\pm$0.31 & 76.91$\pm$0.21 & 60.09$\pm$0.39\\\bottomrule
    \end{tabular}}
\end{table*}

\newpage
\section{Robust Accuracy}
For a more comprehensive study of adversarial robustness, we extend Table~\ref{tab:NCA_DN} to include PGD attack results with the same strength as FGSM attacks ($\epsilon=0.002$). One can readily see from Table~\ref{tab:IntCLonDN} that the robust accuracy under PGD attacks of the same magnitude is slightly lower (roughly 2-3\% lower) as PGD is a stronger attack. Nevertheless, the trend is consistent -- the models that exhibit better adversarial robustness w.r.t. FGSM attacks also demonstrate superior adversarial robustness w.r.t. PGD attacks.
\begin{table*}[h!]
    \centering
    \caption{The complete Table~\ref{tab:NCA_DN} (Table~1 right column) with additional PGD accuracy.}
    \label{tab:IntCLonDN}
    \scalebox{1}
    {\begin{tabular}{c|cccccc}
    \toprule
    \multirow{2}{*}{$\M$} & \multicolumn{6}{c}{$\alpha=0,~\mathcal{L}_{\text{NaCl}}(G^1, \M, \lambda)=\mathcal{L}_{\text{NCA}}(g_2, \M)$} \\
      &  CIFAR100 Acc. & FGSM Acc. & PGD Acc. & CIFAR10 Acc. & FGSM Acc. & PGD Acc. \\\midrule
      1 & 56.83$\pm$0.20 & 31.03$\pm$0.41 & 28.80$\pm$0.48 & 77.24$\pm$0.29 & 48.38$\pm$0.70 & \textbf{46.24$\pm$0.77}\\
      2 & 57.87$\pm$0.15 & 32.50$\pm$0.48 & 30.25$\pm$0.60 & 77.43$\pm$0.11 & 48.14$\pm$0.31 & 45.81$\pm$0.43\\
      3 & 58.42$\pm$0.23 & \textbf{33.19$\pm$0.60} & \textbf{30.93$\pm$0.59} & 77.41$\pm$0.17 & 48.09$\pm$0.93 & 45.67$\pm$0.93\\
      4 & \textbf{58.86$\pm$0.18} & 32.65$\pm$1.07 & 30.22$\pm$1.09 & 77.46$\pm$0.29 & \textbf{48.43$\pm$0.94} & 45.99$\pm$1.15\\
      5 & 58.81$\pm$0.21 & 32.86$\pm$0.47 & 30.57$\pm$0.55 & \textbf{77.58$\pm$0.23} & 48.30$\pm$0.39 & 45.80$\pm$0.48\\\midrule
      & \multicolumn{6}{c}{$\alpha=0,~\mathcal{L}_{\text{NaCl}}(G^1, \M, \lambda)=\mathcal{L}_{\text{MIXNCA}}(g_2, \M, 0.5)$} \\\midrule
      1 & 56.83$\pm$0.20 & 31.03$\pm$0.41 & 28.80$\pm$0.48 & 77.24$\pm$0.29 & 48.38$\pm$0.70 & 46.24$\pm$0.77 \\
      2 & 59.41$\pm$0.19 & \textbf{32.22$\pm$0.35} & \textbf{30.11$\pm$0.43} & 79.36$\pm$0.65 & 48.86$\pm$0.34 & 46.67$\pm$0.40\\
      3 & 59.81$\pm$0.25 & 32.04$\pm$0.67 & 29.87$\pm$0.65 & 79.41$\pm$0.17 & 48.91$\pm$0.81 & 46.61$\pm$0.86\\
      4 & 59.75$\pm$0.33 & 32.03$\pm$0.34 & 29.85$\pm$0.36 & 79.42$\pm$0.18 & \textbf{49.05$\pm$0.71} & \textbf{46.70$\pm$0.80}\\
      5 & \textbf{59.85$\pm$0.30} & 32.06$\pm$0.72 & 29.99$\pm$0.76 & \textbf{79.45$\pm$0.20} & 48.32$\pm$0.70 & 45.89$\pm$0.82\\
      \midrule 
      & \multicolumn{6}{c}{$\alpha=0,~\mathcal{L}_{\text{NaCl}}(G^1, \M, \lambda)=\mathcal{L}_{\text{MIXNCA}}(g_2, \M, 0.6)$} \\\midrule
      1 & 56.83$\pm$0.20 & 31.03$\pm$0.41 & 28.80$\pm$0.48 & 77.24$\pm$0.29 & 48.38$\pm$0.70 & 46.24$\pm$0.77 \\
      2 & 58.94$\pm$0.29 & 32.65$\pm$0.36 & 30.16$\pm$0.27 & 78.67$\pm$0.15 & \textbf{49.86$\pm$0.59} & \textbf{47.38$\pm$0.70}\\
      3 & 59.43$\pm$0.35 & 32.91$\pm$0.40 & 30.36$\pm$0.52 & 78.94$\pm$0.19 & 48.84$\pm$1.09 & 46.24$\pm$1.32\\
      4 & \textbf{59.54$\pm$0.28} & 33.02$\pm$0.62 & \textbf{30.68$\pm$0.72} & 78.92$\pm$0.29 & 49.64$\pm$0.74 & 47.15$\pm$0.88\\
      5 & 59.52$\pm$0.28 & \textbf{33.10$\pm$0.50} & 30.63$\pm$0.48 & \textbf{79.29$\pm$0.21} & 49.39$\pm$1.02 & 46.89$\pm$1.12\\\midrule
      & \multicolumn{6}{c}{$\alpha=0,~\mathcal{L}_{\text{NaCl}}(G^1, \M, \lambda)=\mathcal{L}_{\text{MIXNCA}}(g_2, \M, 0.7)$} \\\midrule
      1 & 56.83$\pm$0.20 & 31.03$\pm$0.41 & 28.80$\pm$0.48 & 77.24$\pm$0.29 & 48.38$\pm$0.70 & 46.24$\pm$0.77 \\
      2 & 58.24$\pm$0.19 & 33.24$\pm$0.90 & 30.40$\pm$1.06 & 78.30$\pm$0.31 & \textbf{50.40$\pm$0.83} & \textbf{47.50$\pm$0.89}\\
      3 & 58.74$\pm$0.26 & 33.12$\pm$0.59 & 29.94$\pm$0.62 & 78.49$\pm$0.30 & 49.85$\pm$0.38 & 46.69$\pm$0.32\\
      4 & 58.79$\pm$0.38 & \textbf{33.63$\pm$0.53} & \textbf{30.70$\pm$0.60} & 78.51$\pm$0.29 & 49.88$\pm$0.75 & 47.01$\pm$0.96\\
      5 & \textbf{58.99$\pm$0.18} & 32.93$\pm$0.81 & 29.89$\pm$0.99 & \textbf{78.57$\pm$0.12} & 49.53$\pm$1.55 & 46.41$\pm$1.91\\\midrule 
      & \multicolumn{6}{c}{$\alpha=0,~\mathcal{L}_{\text{NaCl}}(G^1, \M, \lambda)=\mathcal{L}_{\text{MIXNCA}}(g_2, \M, 0.8)$} \\\midrule
      1 & 56.83$\pm$0.20 & 31.03$\pm$0.41 & 28.80$\pm$0.48 & 77.24$\pm$0.29 & 48.38$\pm$0.70 & 46.24$\pm$0.77 \\
      2 & 57.60$\pm$0.15 & 34.14$\pm$0.22 & 31.35$\pm$0.25 & \textbf{77.96$\pm$0.07} & \textbf{51.82$\pm$0.68} & \textbf{48.81$\pm$0.85}\\
      3 & 58.04$\pm$0.28 & 33.93$\pm$0.45 & 31.31$\pm$0.62 & 77.55$\pm$0.18 & 50.30$\pm$0.81 & 47.41$\pm$0.76\\
      4 & 58.05$\pm$0.16 & \textbf{34.16$\pm$0.54} & \textbf{31.41$\pm$0.61} & 77.90$\pm$0.21 & 50.40$\pm$0.43 & 47.58$\pm$0.47\\
      5 & \textbf{58.43$\pm$0.27} & 33.87$\pm$0.62 & 31.23$\pm$0.76 & 77.90$\pm$0.17 & 50.78$\pm$0.95 & 47.96$\pm$1.12\\\midrule
      & \multicolumn{6}{c}{$\alpha=0,~\mathcal{L}_{\text{NaCl}}(G^1, \M, \lambda)=\mathcal{L}_{\text{MIXNCA}}(g_2, \M, 0.9)$} \\\midrule
      1 & 56.83$\pm$0.20 & 31.03$\pm$0.41 & 28.80$\pm$0.48 & 77.24$\pm$0.29 & 48.38$\pm$0.70 & 46.24$\pm$0.77 \\
      2 & 57.16$\pm$0.15 & 34.25$\pm$0.55 & 31.83$\pm$0.57 & 77.19$\pm$0.09 & \textbf{51.42$\pm$0.45} & \textbf{49.09$\pm$0.53}\\
      3 & 57.08$\pm$0.10 & 33.96$\pm$0.19 & 31.56$\pm$0.34 & 77.21$\pm$0.26 & 51.30$\pm$1.05 & 48.60$\pm$1.28\\
      4 & 57.36$\pm$0.19 & \textbf{34.29$\pm$0.15} & \textbf{31.93$\pm$0.32} & \textbf{77.34$\pm$0.34} & 51.16$\pm$0.55 & 48.64$\pm$0.61\\
      5 & \textbf{57.38$\pm$0.16} & 34.25$\pm$0.30 & 31.89$\pm$0.26 & 77.13$\pm$0.16 & 50.68$\pm$0.74 & 48.14$\pm$0.83\\
      \bottomrule
    \end{tabular}}
\end{table*}

In Figure~\ref{fig:robust_epsilon}, we show the robust accuracy as a function of the FGSM attack strength $\epsilon$. Specifically, we range the attack strength from $0.002$ to $0.032$ and give the robust accuracy of our proposals (IntCl \& IntNaCl) together with baselines under all attacks. From Figure~\ref{fig:robust_epsilon}, one can see that among all baselines, Adv demonstrates the best adversarial robustness, whereas our proposals still consistently win over it by a noticeable margin.
\begin{figure}[h!]
    \centering
    \includegraphics[height=6cm]{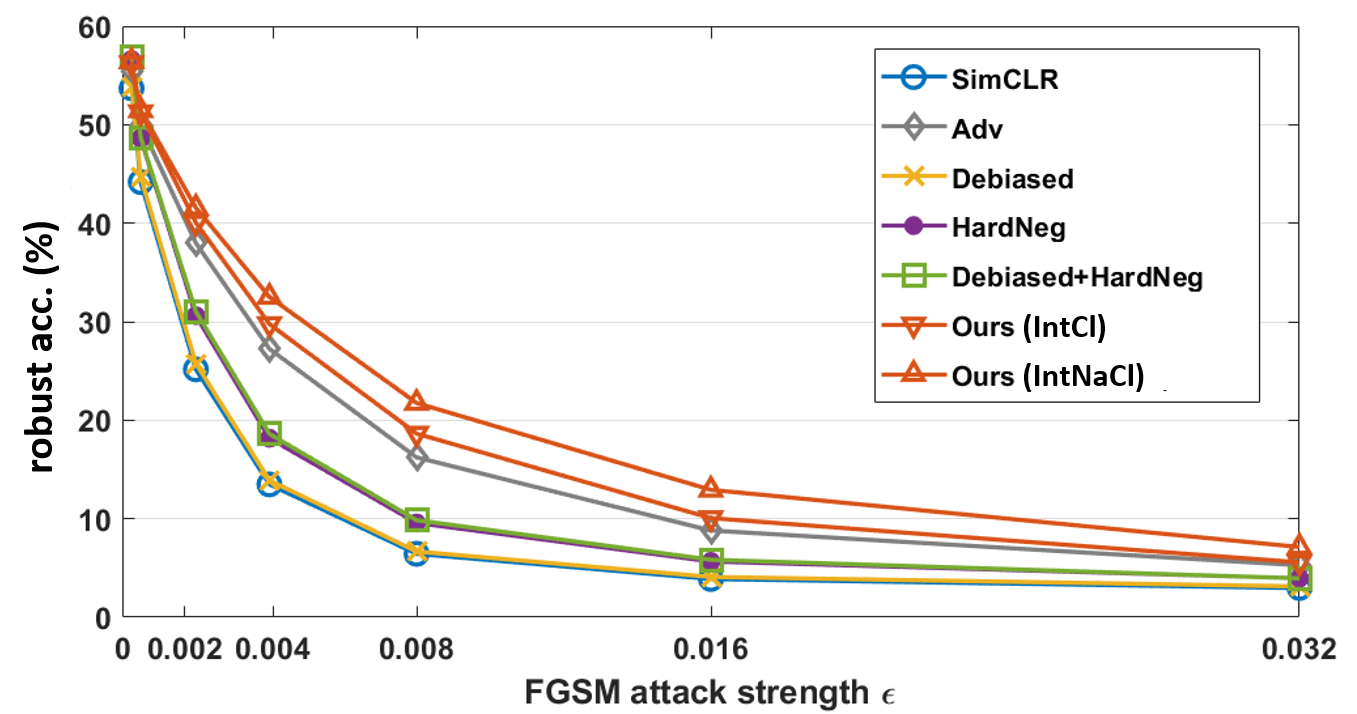}
    \caption{The robust accuracy under FGSM attacks of different strength on CIFAR100.}
    \label{fig:robust_epsilon}
\end{figure}

\newpage
\section{The Effect of $\lambda$}
\begin{figure*}[h!]
    \centering
    \begin{subfigure}[NaCl on SimCLR~\cite{chen2020simple}, i.e. $\alpha=0, \mathcal{L}_{\text{NaCl}}=\mathcal{L}_{\text{MIXNCA}}, G^1=g_0$ in Eq.~\eqref{eqn:IntNCACL}]{\includegraphics[width=\textwidth]{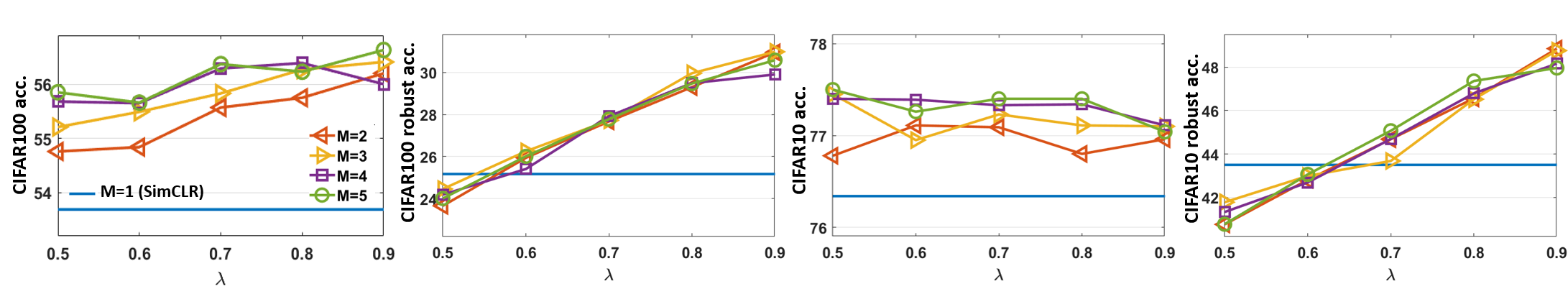}
    \label{fig:ori_tbl1}}
    \end{subfigure}
    \begin{subfigure}[NaCl on Debiased+HardNeg~\cite{robinson2021contrastive}, i.e. $\alpha=0, \mathcal{L}_{\text{NaCl}}=\mathcal{L}_{\text{MIXNCA}}, G^1=g_2$ in Eq.~\eqref{eqn:IntNCACL}]{\includegraphics[width=\textwidth]{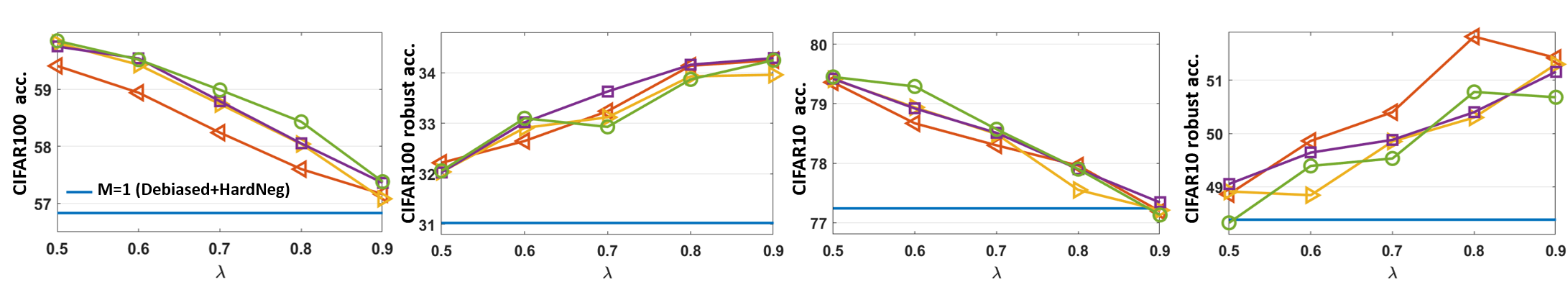}
    \label{fig:ori_tbl2}}
    \end{subfigure}
        \caption{The standard and robust accuracy (\%) on CIFAR100 and CIFAR10 as functions of $\lambda$ in Eq.~\eqref{eqn:IntNCACL} when $\alpha=0,\mathcal{L}_{\text{NaCl}}=\mathcal{L}_{\text{MIXNCA}}$.  }
    \label{fig:MIXNCA_lam}
\end{figure*}
%

\newpage
\section{Extended Runtime}
As training the representation with more epochs can also expose the data to more augmentations, we carry out an additional experiments to compare the efficiency and ultimate accuracy of $\mathcal{L}_{\text{NaCl}}$, $\mathcal{L}_{\text{SimCLR}}$, and $\mathcal{L}_{\text{Debiased+HardNeg}}$.
In Table~\ref{tab:more_epochs_}, we give the standard accuracy of NaCl on SimCLR and NaCl on Debiased+HardNeg at different epochs. Same as before, we only select one $\lambda$ when $\mathcal{L}_{\text{NaCl}}=\mathcal{L}_{\text{MIXNCA}}$ and report its results together with those of $\mathcal{L}_{\text{NaCl}}=\mathcal{L}_{\text{NCA}}$.
In Figure~\ref{fig:more_epochs}, we plot the best standard accuracy achieved as a function of training epochs.
Specially,~\cite{haochen2021provable} has reported a $\mathcal{L}_{\text{SimCLR}}$ CIFAR100 accuracy of 54.74\% after 200 epochs, compared to $\mathcal{L}_{\text{NCA}}(g_0,2)$’s 55.72\% after 100 epochs.
In our reproduction of the $\mathcal{L}_{\text{SimCLR}}$ 200-epoch result\footnote{We let the dataloader shuffle the whole dataset to form new batches after every epoch, so by doubling the training epoch, one will effectively expose the network to more diverse negative pairs.}, we have witnessed an accuracy of 57.45\% however at the cost of 1.34X training time (cf. 200 epochs with $\mathcal{L}_{\text{SimCLR}}$ takes 211 mins vs. 100 epochs with $\mathcal{L}_{\text{NCA}}(g_0,2)$ takes 158 mins).
Overall, we see that NaCl methods demonstrate better efficiency when applying on SimCLR and better ultimate accuracy when applying on Debiased+HardNeg.
\begin{table}[th]
    \centering
    \scalebox{0.9}
    {
    \begin{tabular}{cccccccccccc}
    \#epoch & 100 & 200 & 400 & 600 & 800 & 1000 & 1200 & 1400 & 1600 & 1800 & 2000\\\hline
    $\mathcal{L}_{\text{SimCLR}}$ & 53.69 & 57.45 &	60.06	& 60.96 & 61.27	& 61.90 & 61.94 & 62.53 & 62.44 & 62.10 & 62.06 \\
    $\mathcal{L}_{\text{NCA}}(g_0,2)$ & 55.72 & 59.31  & 61.19 & 61.66 & 62.49 & 61.95 & 62.06 & 62.39 & 62.39 & 62.52 & 62.54\\
    $\mathcal{L}_{\text{MIXNCA}}(g_0,2,0.9)$ & 56.20 & 58.98	& 61.81 & 62.43 & 62.46 & 63.48 & 63.48 & 64.13 & 64.14 & 64.21 & 64.31\\\hline
    $\mathcal{L}_{\text{Debiased+HardNeg}}$ & 56.83 & 59.35 & 61.77 &	62.74	& 62.68 & 63.12 & 63.22 & 63.08 & 62.86 & 62.90 & 63.38\\
    $\mathcal{L}_{\text{NCA}}(g_2,2)$ & 57.87 & 60.06 & 62.36 & 62.58 & 62.86 & 63.07 & 63.29 & 63.65 & 63.13 & 63.73 & 63.20\\
    $\mathcal{L}_{\text{MIXNCA}}(g_2,2,0.5)$ & 59.41 & 62.14	& 64.06 & 65.59 & 65.53 & 66.29 & 66.64 & 67.14 & 66.94 & 67.53 & 67.85\\
    \end{tabular}}
    \caption{The CIFAR100 linear evaluation results (\%) after different numbers of training epochs.}
    \label{tab:more_epochs_}
\end{table}
\begin{figure}[h!]
    \centering
    \begin{subfigure}[NaCl on SimCLR~\cite{chen2020simple}]{
    \includegraphics[height=4cm]{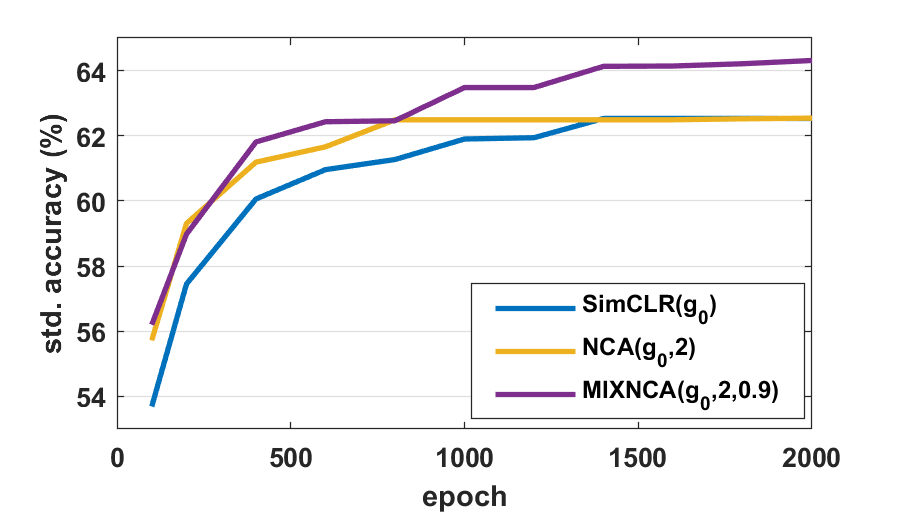}}
    \end{subfigure}
    \begin{subfigure}[NaCl on Debiased+HardNeg~\cite{robinson2021contrastive}]{
    \includegraphics[height=4cm]{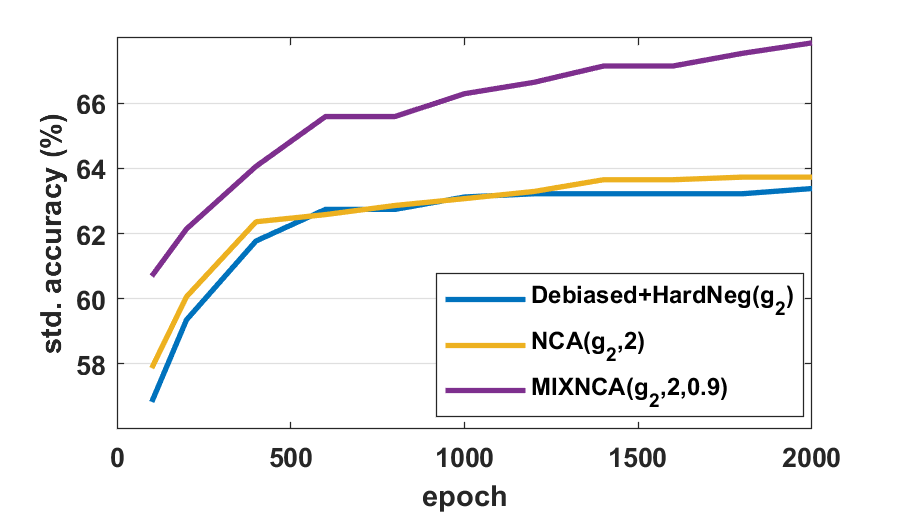}}
    \end{subfigure}
    \caption{The standard accuracy (\%) on CIFAR100 with extended runtime.}
    \label{fig:more_epochs}
\end{figure}

\newpage
\section{Experimental Details}
\paragraph{Architecture.} We follow~\cite{chen2020simple,robinson2021contrastive} to incorporate an MLP projection head during the contrastive learning on resnet18.
\paragraph{Optimizer.} Adam optimizer with a learning rate of $3e-4$.
\paragraph{Training epochs.} The representation network is trained for 100 epochs. For CIFAR100 and CIFAR10, the downstream fully-connected layer is trained for 1000 epochs. For TinyImagenet, the fully-connected layer is trained for 200 epochs.
\paragraph{Methodological hyperparameters.} Throughout out experiments, we use $\tau^+=0.01$ and $\beta=1.0$ for $\mathcal{L}_{\text{Debiased}}$~\cite{chuang2020debiased} and $\mathcal{L}_{\text{Debiased+HardNeg}}$~\cite{robinson2021contrastive}, $\alpha=1$ for $\mathcal{L}_{\text{Adv}}$~\cite{ho2020contrastive}. The same set of hyperparameters are used in our IntCl and IntNaCl.
\paragraph{Data augmentation.} Our data augmentation includes random resized crop, random horizontal flip, random grayscale, and color jitter. Specifically, we implement the color jitter by calling $torchvision.transforms.ColorJitter(0.8 * s, 0.8 * s, 0.8 * s, 0.2 * s)$ and execute with probability $0.8$. Random grayscale is performed with probability $0.2$.
\paragraph{Adversarial hyperparameters.} When evaluating the adversarial robustness using the codebase provided in~\cite{Wong2020Fast}, we use a PGD step size of $1e-2$, $10$ iterations, and $2$ random restarts. 
\paragraph{Error bar.} We run five independent trials for each of the experiments and report the mean and standard deviation for all tables and figures. The error bars in Figure~\ref{fig:robust_epsilon} is omitted for better visual clarity.

\newpage
\section{Supervised Learning Baseline}
We give in the following the standard and robust accuracy of a supervised learning baseline with the same network architecture, optimizer, and batch size. In our self-supervised representation learning experiments, we train the representation network for 100 epochs and train the downstream fully-connected classifying layer for 1000 epochs. Therefore, to obtain a fair supervised learning baseline, we train the complete network end-to-end for 1000 epochs. We follow the same procedures in evaluating the transfer standard accuracy and robust accuracy as described in Section~\ref{sec:exp}.

CIFAR100 (std. acc., FGSM acc., PGD acc.): 65.16$\pm$0.32, 35.89$\pm$0.23, 32.62$\pm$0.23.

Transfer CIFAR10 (std. acc., FGSM acc., PGD acc.): 77.45$\pm$0.21, 44.39$\pm$0.47, 40.35$\pm$0.52.


\end{document}